\documentclass{article}

    \PassOptionsToPackage{numbers, compress}{natbib}
\newcommand{\ourmethod}{\text{OTQMS}}
\newcommand{\allsource}{\text{AllSources $\cup$ Target}}

    \usepackage[final]{neurips_2025}

\usepackage{amssymb}%amssymb 包应该在其他相关包之前加载，尤其是 amsmath

\usepackage[utf8]{inputenc} % allow utf-8 input
\usepackage[T1]{fontenc}    % use 8-bit T1 fonts
\usepackage{hyperref}       % hyperlinks
\usepackage{url}            % simple URL typesetting
\usepackage{booktabs}       % professional-quality tables
\usepackage{amsfonts}       % blackboard math symbols
\usepackage{nicefrac}       % compact symbols for 1/2, etc.
\usepackage{microtype}      % microtypography
\usepackage{xcolor}         % colors
\usepackage{amsmath}
\usepackage{mathtools}
\usepackage{amsthm}
\usepackage{bbding}
\usepackage{float}

\usepackage[capitalize,noabbrev]{cleveref}

\newcommand{\mycomment}[1]{\hspace{\fill} \textcolor{black}{// #1}}

\usepackage{subcaption}
\usepackage{pifont}
\usepackage{booktabs}
\usepackage{float}
\usepackage{graphicx}
\usepackage{makecell}
\usepackage{tikz}
\usepackage{multirow}

\definecolor{MADAvits}{rgb}{0.918, 0, 0}
\definecolor{MADAres50}{rgb}{0, 0.918 ,0.918}
\definecolor{allsources}{rgb}{0.290, 0.486, 0.192}
\definecolor{ours}{rgb}{0, 0, 0.918}

\DeclareMathAlphabet{\mathbbb}{U}{bbold}{m}{n}  % bb bold numbers

\newtheorem{theorem}{Theorem}
\newtheorem{lemma}[theorem]{Lemma}
\newtheorem{proposition}[theorem]{Proposition}

\newtheorem{definition}[theorem]{Definition}

\DeclareMathOperator*{\argmax}{arg\,max}
\DeclareMathOperator*{\argmin}{arg\,min}

\newcommand\tsup[2][2]{%
	\def\useanchorwidth{T}%
	\ifnum#1>1%
	\stackon[-1.3ex]{\tsup[\numexpr#1-1\relax]{#2}}{\mathchar"307E\kern-.5pt}%
	\else%
	\stackon[-1ex]{#2}{\mathchar"307E\kern-.5pt}%
	\fi%
}

\newcommand{\cT}{{\mathcal{T}}}

\newcommand{\X}{{\mathcal{X}}}

\newcommand{\cA}{\mathcal{A}}

\newcommand{\cL}{\mathcal{L}}

\newcommand{\cS}{\mathcal{S}}

\DeclareMathOperator{\var}{var}

\newcommand{\defeq}{\triangleq}

\usepackage[textsize=tiny]{todonotes}

\usepackage{algorithm}
\usepackage{algorithmic}

\usepackage{adjustbox}
\usepackage{wrapfig}

\title{
A High-Dimensional Statistical Method for Optimizing Transfer Quantities in Multi-Source Transfer Learning}

\author{%
  \parbox{\textwidth}{\centering
  Qingyue Zhang\textsuperscript{1}\thanks{Equal contribution.},\;
  Haohao Fu\textsuperscript{1}\footnotemark[1],\;
  Guanbo Huang\textsuperscript{1}\footnotemark[1],\;
  Yaoyuan Liang\textsuperscript{1},\;
  Chang Chu\textsuperscript{1},\\[2pt]
  Tianren Peng\textsuperscript{1},\;
  Yanru Wu\textsuperscript{1},\;
  Qi Li\textsuperscript{1},\;
  Yang Li\textsuperscript{1}\thanks{Corresponding authors: Yang Li (yangli@sz.tsinghua.edu.cn), Shao-Lun Huang(twn2gold@gmail.com)},\;
  Shao-Lun Huang\textsuperscript{1}\footnotemark[2]\\[4pt]
  {\normalfont\textsuperscript{1}Tsinghua Shenzhen International Graduate School, Tsinghua University}
  }
}

\begin{document}

\maketitle

\begin{abstract}
Multi-source transfer learning provides an effective solution to data scarcity in real-world supervised learning scenarios by leveraging multiple source tasks.
In this field, existing works typically use all available samples from sources in training, which constrains their training efficiency and may lead to suboptimal results.
To address this, we propose 
a theoretical framework that answers the question: what is the optimal quantity of source samples needed from each source task to jointly train the target model? 
Specifically, we introduce a generalization error measure based on K-L divergence, and minimize it based on high-dimensional statistical analysis
to determine the optimal transfer quantity for each source task.
Additionally, we develop an architecture-agnostic and data-efficient algorithm \ourmethod{} to implement our theoretical results 
for target model training in multi-source transfer learning.  
Experimental studies on diverse architectures and two real-world benchmark datasets show that our proposed algorithm significantly outperforms state-of-the-art approaches in both accuracy and data efficiency. The code
is available at \url{https://github.com/zqy0126/OTQMS}.
\end{abstract}

\section{Introduction}
\label{Introduction}

\begin{wrapfigure}{r}{0.50\linewidth}
    \centering
    \vspace{-2.6em}
    ~\kern-.35em\includegraphics[width=1.00\linewidth]{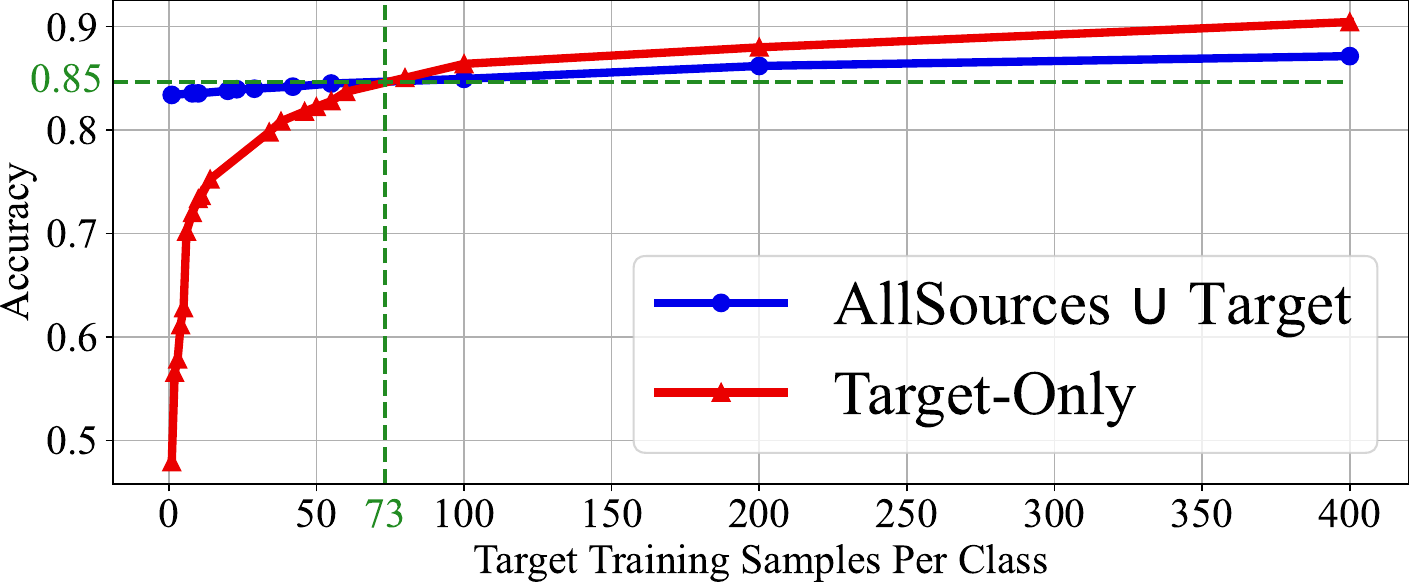}
    \vspace{-1.85em}
    \caption{\textbf{More source samples does not always mean better performance.} Incorporating all source samples may bring negative impact, 
    which is illustrated by the comparison of 
    two strategies, using target task samples with all source samples (\textcolor{blue}{blue}) and using target task samples only (\textcolor{red}{red}), 
    evaluated on the equally divided 5-task CIFAR10 dataset. Theoretically, although incorporating more source samples reduces model variance by expanding the training data, the discrepancy between the source and target tasks introduces additional bias.} 
    \label{fig:teaser-cifar10}
    \vskip-1em
\end{wrapfigure}

Nowadays, various machine learning algorithms have achieved remarkable success by leveraging large-scale labeled training data. However, in many practical scenarios, the limited availability of labeled data presents a significant challenge, where transfer learning emerges as an effective solution \cite{zhao2020review}. Transfer learning aims to leverage knowledge from tasks with abundant data or being well-trained, known as the source tasks, to improve the performance of a new learning task, known as the target task. Given its numerous applications, transfer learning has gained wide
popularity and seen success in a variety of fields, such as computer vision \cite{wang2018deep}, natural language processing \cite{sung2022vl}, recommendation systems \cite{fu2024exploring} and anomaly detection \cite{vincent2020transfer}. Traditionally, transfer learning has focused on the transfer between a single source task and a target task. However, there is a growing emphasis on multi-source transfer learning, which leverages multiple source tasks to enhance the training of the target task \cite{sun2015survey}.

% \begin{figure}[t]
%     \centering    \includegraphics[width=0.85\linewidth]{hgb_1_30_figure/teaser_1_30.pdf}
%     \caption{\textbf{More source samples does not always mean better performance.} Incorporating all source samples may bring negative impact, 
%     which is illustrated by the comparison of 
%     two strategies, using target with all source samples (\textcolor{blue}{blue}) and using target only (\textcolor{red}{red}), 
%     evaluated on the equally divided 5-task CIFAR10 dataset. Theoretically, although incorporating more source samples reduces model variance by expanding the training data, the discrepancy between the source and target tasks introduces additional bias.} 
%     \label{fig:teaser-cifar10}
%     \vspace{-6mm} 
% \end{figure}

In multi-source transfer learning, traditional methods usually jointly train the target model using all available samples from sources without selection~\cite{zhang2024revisiting_MADA,shui2021aggregating_WADN,li2021dynamic}. This evidently poses a severe limitation to training efficiency, considering the vast number of available samples from various potential source tasks in real-world scenarios~\cite{peng2019moment_M3SDA}. Moreover, directly assuming the use of all available samples seriously constrains their solution space, which possibly leads to suboptimal results, as illustrated in Figure~\ref{fig:teaser-cifar10}. %Therefore, it is meaningful to establish a theoretical framework and corresponding practical algorithm to get the optimal quantity of source samples needed from each source task in training.
Therefore, it is critical to establish a theoretical framework to answer the question: \textit{what is the optimal transfer quantity of samples from each source task needed in training the target task?}

In this work, 
% we establish a theoretical framework based on the multi-source transfer learning problem. 
we formulate a sample-based multi-source transfer learning problem as a parameter estimation problem. By employing high-dimensional statistical methods to analyze it, we establish a theoretical framework to determine the optimal transfer quantity for each source task.
Specifically, we introduce the expectation of Kullback-Leibler (K-L) divergence between the true distribution of target task samples and the distribution learned from training samples as a measure of generalization error. This measure is then minimized in the asymptotic regime to derive the optimal transfer quantity of each source task.
Building on this, we propose a practical algorithm, named \textbf{O}ptimal \textbf{T}ransfer \textbf{Q}uantities for \textbf{M}ulti-\textbf{S}ource Transfer learning (\ourmethod{}), to implement our theoretical results for training multi-source transfer learning models. Notably, \ourmethod{} is data-efficient and compatible with various model architectures, including Vision Transformer (ViT) and Low-Rank Adaptation (LoRA). It also demonstrates advantages in \emph{task generality} and \emph{shot generality} (as illustrated in Table~\ref{tab:comparison}), since we establish theoretical framework without restricting it to a specific target task type or limiting the range of target sample quantity. In summary, our main contributions are as follows:
\begin{itemize}
\item[a)]
We use high-dimensional statistics to analyze the parameter estimation problem formulated by the sample-based multi-source transfer learning problem. Based on this, a novel theoretical framework that optimizes transfer quantities is introduced.
% We propose a new K-L divergence measure of the generalization error. By
% minimizing it, we develop a method to determine the optimal transfer quantity of each source task.

\item[b)] Based on the framework, we propose \ourmethod{}, an architecture-agnostic and data-efficient algorithm for target model training
in multi-source transfer learning. In particular, we propose a dynamic strategy in \ourmethod{} to alleviate the estimation error of transfer quantities caused by the scarcity of target task samples.
%, which is appropriate for
% \item[b)]Moreover, we design an practical and data-efficient algorithm to validate our theory. 
%The  experimental framework is designed to be lightweight and easy to interpret and adapt, possible to serve as a \textit{plug-in} module. 

\item[c)]Experimental studies on few-shot multi-source transfer learning tasks, two real-world datasets and various model architectures demonstrate that \ourmethod{} achieves significant improvements in both accuracy and data efficiency. In terms of accuracy,  \ourmethod{} outperforms state-of-the-art approaches by an average of $1.5\%$ on \texttt{DomainNet} and $1.0\%$ on \texttt{Office-Home}. 
In terms of data efficiency, \ourmethod{} reduces the average training time by $35.19\%$ and the average sample usage by $47.85\%$ on \texttt{DomainNet}. 
Furthermore, extensive supplementary experiments demonstrate that \ourmethod{} can facilitate both full model training and parameter-efficient training, and \ourmethod{} is also applicable to multi-task learning tasks. 
% Extensive supplementary experiments present the generality of our methodology, and provide an analysis of the transfer quantity.
\end{itemize}
\begin{table*}[t]
\setlength{\tabcolsep}{4pt}
    \renewcommand{\arraystretch}{0.76}
    \centering
    \caption{\textbf{Comparison across matching-based transfer learning method,} based on whether they are tailored to multi-sources, have task generality, have shot generality, and require target labels. The `\Checkmark' represents obtaining the corresponding aspects, while `\XSolidBrush' the opposite. Task generality denotes the ability to handle various target task types, and shot generality denotes the ability to avoid negative transfer in different target sample quantity settings including few-shot and non-few-shot.}
    \begin{tabular}{c c c c c}
        \toprule
         \textbf{Method} &  \textbf{Multi-Source}& \textbf{Task Generality}  & \textbf{Shot Generality} & \textbf{Target Label} \\
         \midrule
         % MTL \citep[]{sun2019meta} &  \XSolidBrush &  \Checkmark & \XSolidBrush & Few-shot \\
         MCW \citep[]{lee2019learning_MCW} &  \Checkmark & \XSolidBrush & \XSolidBrush & Supervised \\
         Leep \citep[]{nguyen2020leep} &  \XSolidBrush &  \Checkmark & \Checkmark & Supervised \\
         Tong \citep[]{tong2021mathematical} &  \Checkmark &  \XSolidBrush & \Checkmark &Supervised \\
         DATE \citep[]{han2023discriminability_DATE} & \Checkmark   &  \Checkmark & \Checkmark & Unsupervised \\
         H-ensemble \citep[]{wu2024h_Hensemble}  &  \Checkmark & \XSolidBrush & \XSolidBrush & Supervised \\ 
         \ourmethod{} (Ours)   &  \Checkmark &  \Checkmark &  \Checkmark & Supervised \\
         \bottomrule
    \end{tabular}
    \label{tab:comparison}
    % \vspace{-2em}
\end{table*}

\section{Related Work}
This work is related to two main lines of research. The first is transfer learning theory, where we propose a measure based on K-L divergence as a novel measure of generalization error, different from previously adopted measures ~\cite{harremoes2011pairs,bu2020tightening,tong2021mathematical,bao2019information,wu2024h_Hensemble}. The second is multi-source transfer learning, where most existing studies either utilize all samples from all source tasks or adopt task-level selection strategies~\cite{guo2020multi,shui2021aggregating_WADN,tong2021mathematical,wu2024h_Hensemble}. In contrast, our framework explicitly optimizes the transfer quantity from each individual source task. Other closely related work is discussed further in Appendix \ref{appendix_related_work}.

\section{Problem Formulation}
% Let X be the random variable denoting the data with input space $\mathcal{X}$. 

Consider the transfer learning setting with one target task $\cT$, and $K$ source tasks $\left\{\cS_1,\dots,\cS_K\right\}$. 
The target task $\mathcal{T}$ is not restricted to a specific downstream task category. Generally, we formulate it as a parameter estimation problem under a distribution model 
$P_{X;{\underline{\theta}}}$. For example, when $\cT$ is a supervised classification task, $P_{X;{\underline{\theta}}}$ corresponds to the joint distribution model of input features $Z$ and output labels $Y$, \textit{i.e.}, $X=(Z,Y)$. Our objective is to estimate the true value of $\underline{\theta}$, which corresponds to optimizing the neural network parameters for target task $\cT$. 
Here, $\theta$ denotes 1-dimensional parameter, and $\underline{\theta}$ denotes high-dimensional parameter. 
% We provide notation table is the Table \ref{Notations} in appendix. 
Furthermore, we assume that the source tasks and the target task follow the same parametric model and share the same input space $\mathcal{X}$.
Without loss of generality, we assume $\mathcal{X}$ to be discrete primarily for clarity of writing and to avoid redundancy, and our results can be readily extended to continuous spaces. The target task $\cT$ has $N_{0}$ training samples $X^{N_0}=\{x_1,\dots,x_{N_{0}}\}$ i.i.d. generated from some underlying joint distribution $P_{X;{\underline{\theta}}_0}$, where the parameter ${\underline{\theta}}_0\in \mathbb{R}^d$. Similarly, the source task $\cS_i$ has $N_i$ training samples $X^{N_i}=\{x^{i}_1, \dots,x^{i}_{N_i}\}$ i.i.d. generated from some underlying joint distribution $P_{X;{\underline{\theta}}_i}$, where $i\in[1,K]$, and the parameter ${\underline{\theta}}_i\in \mathbb{R}^d$. 
% Based on the aforementioned assumptions regarding the parametric model, the neural network training problem for the target task $\cT$ can be equivalently transformed into a parameter estimation problem whose true value is $\theta_0$. 
In this work, we use the Maximum Likelihood Estimator (MLE) to estimate the true target task parameter $\theta_0$.
% , which can be realized by optimizing the cross-entropy using gradient descent during neural network training \cite{du2019gradient}. 
% Moreover, the following lemma demonstrates that, for large sample size, the MLE can achieve the theoretical lower bound of the Mean Squared Error (MSE)  asymptotically, known as the Cram\'er-Rao Bound. 
Moreover, the following lemma characterizes the asymptotic behavior of MLE.

\begin{lemma}\label{thm:Cramér}(Asymptotic Normality of the MLE) \cite{wasserman2013all}
When we use MLE only based on target task samples to estimate $\theta_0$, i.e.
\begin{align}
\label{mle_target}
\hat{{\underline{\theta}}}_{MLE}=\argmax_{{\underline{\theta}}} {\frac{1}{N_{0}}\sum_{x \in X^{N_{0}}}} \log P_{X;{\underline{\theta}}}(x),
\end{align}
under appropriate regularity conditions, the following holds:
\begin{align}
\label{eq:Cramér2}
\sqrt{N_0}\left(\hat{{\underline{\theta}}}_{MLE} - \underline{\theta}_0\right)\xrightarrow{d}\mathcal{N}\left(0,J({\underline{\theta}}_0)^{-1}\right),
\end{align}
where “${-1}$” denotes the matrix inverse and the $J({\underline{\theta}})$ is the Fisher information matrix defined as:
\begin{align}
    J({\underline{\theta}})^{d\times d } &= \mathbb{E} \left[ \left( \frac{\partial}{\partial {\underline{\theta}}} \log P_{X;{\underline{\theta}}} \right) \left( \frac{\partial}{\partial {\underline{\theta}}} \log P_{X;{\underline{\theta}}} \right)^T \right].
\end{align}

\end{lemma}

When we transfer $n_1,\dots,n_K$ samples from the $\left\{\cS_1,\dots,\cS_K\right\}$, where $n_i\in\left[0,N_i\right]$, we denote these training sample sequences as $X^{n_1},X^{n_2},\dots,X^{n_K}$ . Then, the MLE for multi-source transfer learning, is the parameter value that maximizes the empirical mean of the likelihood of samples from source tasks and target task, \textit{i.e.},
\begin{align}
\label{mle_defination}
    \hat{{\underline{\theta}}} = \argmax_{{\underline{\theta}}}\frac{1}{N_{0}+\sum_{i=1}^{K}n_i} \left({\sum_{x \in X^{N_{0}}}} \log P_{X;{\underline{\theta}}}(x)
    +{\sum_{i=1}^{K}\sum\limits_{x \in X^{n_{i}}}} \log P_{X;{\underline{\theta}}}(x)\right). 
\end{align}

In this work, 
our goal is to derive the optimal transfer quantities $n_1^*,\dots,n_K^*$ of source tasks $\cS_1,\dots,\cS_K$ to minimize certain divergence measure $\mathcal{D} iv$ between the true distribution of target task $P_{X;{\underline{\theta}}_0}$ and the distribution $P_{X;\hat{{\underline{\theta}}}}$ learned from training samples, \textit{i.e.},
\begin{align}
\label{optimization_problem}
    n_1^*,\dots,n_K^*=\argmin_{n_1,\dots,n_K} \mathcal{D} iv(P_{X;{\underline{\theta}}_0},P_{X;\hat{{\underline{\theta}}}}).
\end{align}

% In machine learning scenario, it corresponds to the optimal transfer quantities of samples from each source to jointly train the target model.
Besides, we provide the notations table in Appendix \ref{appendix:Notations}. 

\section{Theoretical Analysis and Algorithm}
In this section, we will first introduce a new K-L divergence based measure for the optimization problem in \eqref{optimization_problem}. Then, we will analyze it based on high-dimensional statistics to derive the optimal transfer quantities for both single-source and multi-source scenarios. Finally, we will develop a practical algorithm based on the theoretical framework.
%For the beginning, we will provide their definitions
% In this work, we will utilize the K-L divergence in our measure to evaluate the distance between two distributions.
% In this work, we will utilize the K-L divergence as a measure to evaluate the distance between two distributions. Moreover, Fisher information will also be used in subsequent theoretical derivations, and thus we provide their definitions here.

\begin{definition}(The K-L divergence) \cite{cover1999elements}
The K-L divergence $D\left(P \middle\| Q\right)$ measures the difference between two probability distributions \( P(X) \) and \( Q(X) \) over the same probability space. It is defined as:
\begin{align}
    % D_{\text{KL}}(P \parallel Q) &= \sum_{x \in X} P(x) \log \left( \frac{P(x)}{Q(x)} \right), \quad \text{for discrete distributions,} \\
    % D_{\text{KL}}(P \parallel Q) &= \int_{X} p(x) \log \left( \frac{p(x)}{q(x)} \right) \, dx, \quad \text{for continuous distributions,}
    D\left(P \middle\| Q\right) = \sum_{x \in \mathcal{X}} p(x) \log \frac{p(x)}{q(x)}.\notag
\end{align}
% The K-L divergence is always non-negative and is zero if and only if \( p(x) = q(x) \) for all \( x \in X \). 
% where \( P(x) \) and \( Q(x) \) are the probability mass (or density) functions of the distributions \( P \) and \( Q \), respectively. 
% Moreover, the K-L divergence is a direct correspondence with the cross-entropy loss in model training \cite{cover1999elements}.
% \begin{align}
% D\left(P \middle\| Q\right)&= \sum_{x \in X} p(x) \log{p(x)}\underbrace{-\sum_{x \in X} p(x) \log{q(x)}}_{ cross-entropy}\notag
% %\\\notag&=-H(P)+H(P,Q)\notag
% \end{align}

\end{definition}

In this work, we apply the expectation of the K-L divergence between the true distribution model of the target task and the distribution model learned from training samples to measure the generalization error. Compared to other measures, the K-L divergence exhibits a closer correspondence with the generalization error as measured by the cross-entropy loss,
% as demonstrated in Appendix \ref{appendix:K-L}
.

% which directly corresponds to the testing error using cross-entropy loss \cite{cover1999elements}.
% \begin{align}
%     D\left( P_{X;{\underline{\theta}}_0}\middle\| P_{X;\hat{{\underline{\theta}}}}\right) = -{\underbrace{H(P_{X;\hat{{\underline{\theta}}}})}_{\textnormal{entropy}}}+\underbrace{\mathbb{E}_{P_{X;\hat{{\underline{\theta}}}}}\left[-\log
%     P_{X;{\underline{\theta}}_0}
%     \right]}_{\textnormal{cross-entropy}}\notag.
% \end{align}
Finally, our generalization error measure is defined as:
\begin{align}
    \mathcal{D} iv(P_{X;{\underline{\theta}}_0},P_{X;\hat{{\underline{\theta}}}})=\mathbb{E} \left[ D\left( P_{X;{\underline{\theta}}_0}\middle\| P_{X;\hat{{\underline{\theta}}}}\right) \right] .\label{eq:KL1}
\end{align}

\subsection{Single-Source Transfer Learning} 
\label{Single_Source_Transfer_Learning}
The direct computation of the proposed K-L divergence based measure is challenging.
Fortunately, we can show that the proposed measure is computable in the asymptotic regime using high-dimensional statistical analysis. To be specific, we can prove that the proposed measure directly depends on the Mean Squared Error (MSE) in the asymptotic regime, and the MSE can be calculated by an extension of Lemma \ref{thm:Cramér}. Therefore, we perform an asymptotic analysis of this measure.

In this section, we begin by presenting the theoretical results in Lemma \ref{thm:target_only} and Theorem \ref {thm:one_source} where the parameter is 1-dimensional, and subsequently extend them to the high-dimensional parameter setting in Proposition \ref{Proposition:target_only} and \ref {Proposition:one_source}. To begin, we consider the setting with a target task $\cT$ with $N_0$ samples generated from a model with 1-dimensional parameter.
% the Mean Squared Error (MSE) in the asymptotic case. Moreover, the MSE can be calculated by an extension of Lemma \ref{thm:Cramér}.
% Fortunately, we can prove that the proposed proposed measure is directly depends on
% the Mean Squared Error (MSE) in the asymptotic case. Moreover, the MSE can be calculated by an extension of Lemma \ref{thm:Cramér}.

\begin{lemma}\label{thm:target_only}(proved in Appendix \ref{appendix:target_only})
Given a target task $\cT$ with $N_0$ i.i.d. samples generated from a 1-dimensional underlying model $P_{X;\theta_0}$, where $\theta_0 \in \mathbb{R}$, and denoting $\hat{\theta}$ as the MLE \eqref{mle_target} based on the $N_0$ samples, then the proposed measure \eqref{eq:KL1} can be expressed as: 
\begin{align}
\mathbb{E} \left[ D\left(P_{X;\theta_0}\middle\| P_{X;\hat{\theta}}\right) \right] = \frac{1}{2N_0} + o\left( \frac{1}{N_0} \right).
                                           \end{align}
\end{lemma}

The result of Lemma~\ref{thm:target_only} demonstrates that when there is only one target task without any source task, the proposed measure is inversely proportional to the number of training samples. %This aligns with the empirical intuition that more training samples lead to more accurate training. 
% This also validates the superiority of our setting that uses all $N_0$ samples in the training process.
Next, we consider the transfer learning scenario where we have one target task $\cT$ with $N_0$ training samples and one source task $\cS_1$ with $N_1$ training samples. In this case, we aim to determine the optimal transfer quantity $n_1^*\in [1,N_1]$. To facilitate our mathematical derivations, we assume 
$N_0$ and $N_1$ are asymptotically comparable, and
the distance between the parameters of the target task and source task is sufficiently small (\textit{i.e.}, $\vert\theta_0-\theta_1\vert=O(\frac{1}{\sqrt{N_{0}}})$). Considering the similarity of low-level features among tasks of the same type,
this assumption is made without loss of generality and supported by related studies \cite{raghu2019transfusion}.
Furthermore, as demonstrated in subsequent analysis, our conclusions remain valid even in extreme cases where the distance between parameters is large. 

%to get the best $\hat{\theta}$ which can gain the best result of \eqref{eq:KL1}.

% \begin{theorem}
% \label{thm:one_source}(proved in Appendix \ref{appendix:one_source})
% If we have $N_{0}$ target task samples, and choose $n_1\in[1,N_1]$ source task samples from source task $\cS_1$ to jointly train the target model, where the model parameters is 1-dimension, \textit{i.e.}, $d=1$, the K-L objective function will be
%     \begin{align}
%     \label{thm:one_source_KL}
%     \frac{1}{2}\bigg(\underbrace{\frac{1}{N_{0}+n_1}}_{variance~ term}+\underbrace{\frac{n_1^2}{(N_{0}+n_1)^2}t}_{bias~term}\bigg),
% \end{align}
% where 
% \begin{align}
% t \defeq J(\theta_0)(\theta_1 - \theta_0)^2 .
% \end{align}
% \end{theorem}

\begin{theorem}
\label{thm:one_source}(proved in Appendix \ref{appendix:one_source})
Given a target task $\cT$ with $N_0$ i.i.d. samples generated from a 1-dimensional underlying model $P_{X;\theta_0}$, and a source task $\cS_1$ with $N_1$ i.i.d. samples generated from a 1-dimensional underlying model $P_{X;\theta_1}$, where $\theta_0,\theta_1 \in \mathbb{R}$ and $\vert\theta_0-\theta_1\vert=O(\frac{1}{\sqrt{N_{0}}})$, $\hat{\theta}$ is denoted as the MLE \eqref{mle_defination} based on the $N_0$ samples from $\cT$ and $n_1$ samples from $\cS_1$, where $n_1 \in [0,N_1]$, then the proposed measure \eqref{eq:KL1} can be expressed as: 
    \begin{align}
    \label{thm:one_source_KL}
    % &{\cal L}(P_{X; \hat{\theta}}, P_{X;\theta_0}) \nonumber\\
    % &=
    \frac{1}{2}\bigg(\underbrace{\frac{1}{N_{0}+n_1}}_{\textnormal{variance~ term}}+\underbrace{\frac{n_1^2}{(N_{0}+n_1)^2}t}_{\textnormal{bias~term}}\bigg)+o\left(\frac{1}{N_{0}+n_1}\right),
\end{align}
where 
\begin{align}
t \defeq J(\theta_0)(\theta_1 - \theta_0)^2 .
\end{align}
Moreover, the optimal transfer quantity $n_1^*$ is
\begin{align}
n_1^*=
\begin{cases}
N_1,  &\quad \mathrm{if} \ N_{0} \cdot t\le0.5
 \\
\min\left(N_1,\frac{N_{0}}{2N_{0}t-1}\right), &\quad \mathrm{if} \ N_{0} \cdot t>0.5
\end{cases}.\label{optimalexponent}
\end{align}

\end{theorem}
From \eqref{thm:one_source_KL}, we observe that the proposed measure decreases as $N_{0}$ increases, which aligns with our intuition that utilizing all available target samples is beneficial. In addition, the trend of \eqref{thm:one_source_KL} with respect to $n_1$ is more complex. 
We plot  \eqref{thm:one_source_KL} as a function of $n_1$ under two different regimes, determined by the value of $N_{0} \cdot t$, as shown in Figure \ref{N0*T}.
% We will give the figure of \eqref{thm:one_source_KL} and interpret our results at three different intervals. 
Our goal is to explore the optimal value of $n_1^*$ to minimize \eqref{thm:one_source_KL}.

\begin{wrapfigure}{r}{0.50\linewidth}
\centering
\vspace{-2.6em}
\includegraphics[width=1.00\linewidth]{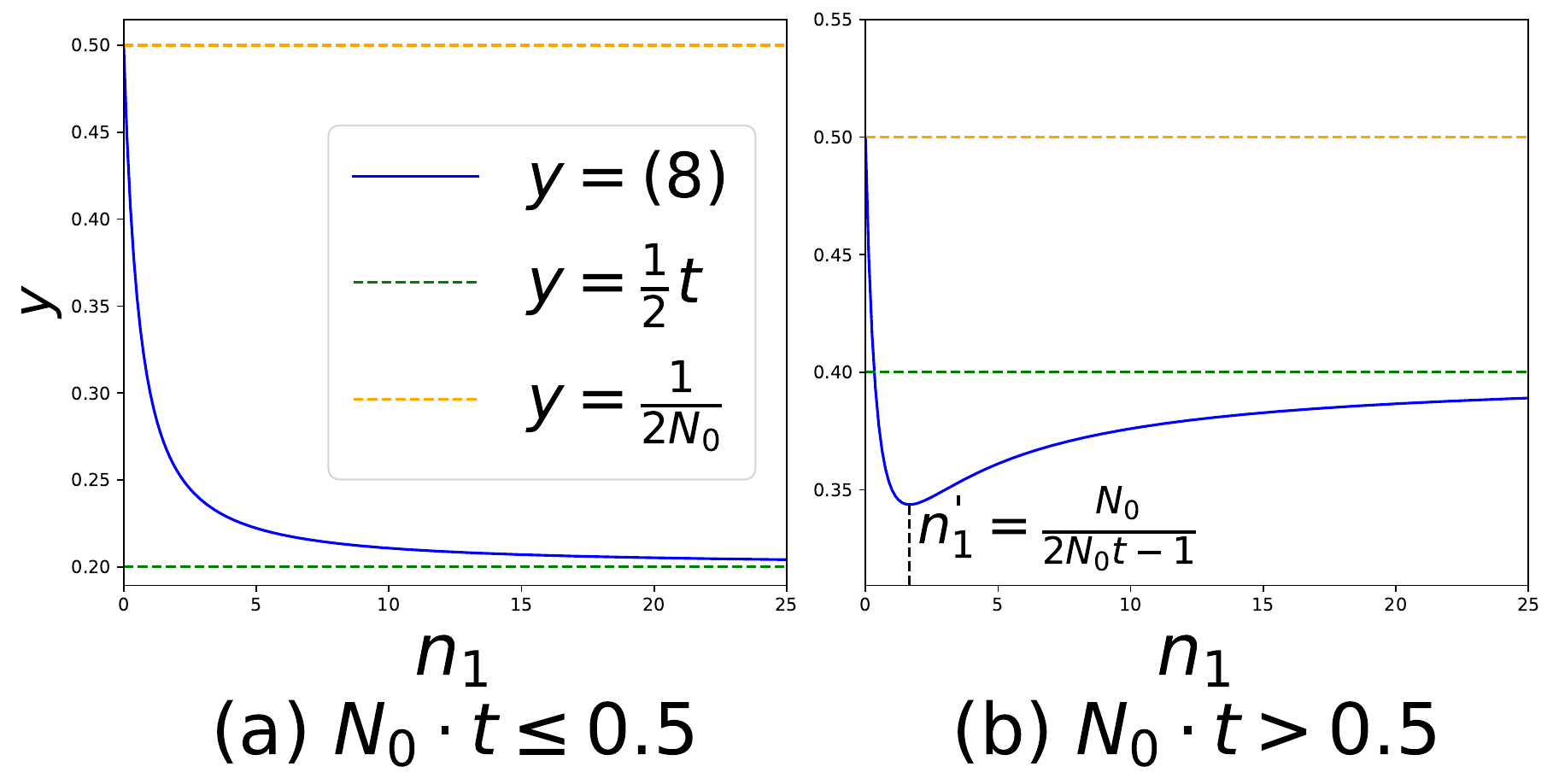}
\caption{The function curve figures of \eqref{thm:one_source_KL} under different regimes determined by the value of $N_{0} \cdot t$ (\textcolor{blue}{blue}). The vertical axis denotes the value of proposed measure \eqref{thm:one_source_KL}, while the horizontal axis denotes the  variable $n_1$.}
\label{N0*T}
\vskip-1em
\end{wrapfigure}

% \begin{wrapfigure}{r}{0.50\linewidth}
%     \centering
%     \vspace{-2.6em}
%     ~\kern-.35em\includegraphics[width=1.00\linewidth]{hgb_1_30_figure/teaser_1_30.pdf}
%     \vspace{-1.85em}
%     \caption{\textbf{More source samples does not always mean better performance.} Incorporating all source samples may bring negative impact, 
%     which is illustrated by the comparison of 
%     two strategies, using target with all source samples (\textcolor{blue}{blue}) and using target only (\textcolor{red}{red}), 
%     evaluated on the equally divided 5-task CIFAR10 dataset. Theoretically, although incorporating more source samples reduces model variance by expanding the training data, the discrepancy between the source and target tasks introduces additional bias.} 
%     \label{fig:teaser-cifar10}
%     \vskip-1em
% \end{wrapfigure}

\begin{itemize}
    \item \textbf{Case 1 ($N_{0} \cdot t\le0.5$):} The proposed measure monotonically decreases as $n_1$ increases. Obviously, the optimal point is $n_1^*=N_{1}$. This indicates that when the source task and the target task are highly similar, \textit{i.e.}, when 
    $t$ is small, an increase in the transfer quantity will positively impact the results. 
    
    \item \textbf{Case 2 ($N_{0}\cdot t>0.5$):} The proposed measure first decreases and then increases as $n_1$ increases. It attaining its minimum at $n_{1}' = \frac{N_{0}}{2N_{0}t-1}$. It should be noted that when $t$ is large enough, the point $n_{1}'$ approaches $0$. This aligns with the intuition that when the discrepancy between the source and target tasks is substantial, avoiding transfer yields better results. Furthermore, if $n_1'$ exceeds $N_1$, we should utilize all $N_1$ samples, so $n_1^*=\min\left(N_1,\frac{N_{0}}{2N_{0}t-1}\right)$.
\end{itemize}

As the dimensionality of the model parameter $\theta$
increases to \textbf{higher dimensions}, we derive the following propositions. 
%which takes a form very similar to the previous Lemma \ref{thm:target_only} and Theorem \ref{thm:one_source}

\begin{proposition}(proved in Appendix \ref{appendix:prop_target_only})
\label{Proposition:target_only}
In the case where the parameter dimension is d, \textit{i.e.}, ${\underline{\theta}}_0\in \mathbb{R}^d$, with all other settings remaining the same as the Lemma  \ref{thm:target_only}, the proposed measure \eqref{eq:KL1} can be expressed as 
\begin{align}
\frac{d}{2N_0}+o\left(\frac{1}{N_0}\right).
\end{align}
\end{proposition}

\begin{proposition}(proved in Appendix \ref{appendix:Proposition:one_source})
\label{Proposition:one_source}
In the case where the parameter dimension is d, \textit{i.e.}, ${\underline{\theta}}_0,{\underline{\theta}}_1\in \mathbb{R}^d$, with all other settings remaining the same as the Theorem  \ref{thm:one_source}, the proposed measure \eqref{eq:KL1} can be expressed as
\begin{align} 
\label{Proposition:one_source_eq}
    %&=\frac{1}{2}\frac{d}{m+n}+\frac{1}{2}tr({J}(\vec{\theta_1})\left(\vec{\theta_2} - \vec{\theta_1}\right)^2)\frac{m^2}{(m+n)^2}\nonumber\\
    % &{\cal L}(P_{X; \hat{\theta}}, P_{X;\theta_0})\nonumber\\
    % &=
    \frac{d}{2}\left(\frac{1}{N_0+n_{1}}+\frac{n_{1}^2}{(N_0+n_{1})^2}t^{}\right)+o\left(\frac{1}{N_{0}+n_{1}}\right),
\end{align}
where we denote
\begin{align}
    t^{} \defeq \frac{({{\underline{\theta}}_1} - {{\underline{\theta}}_0})^{T}J({{\underline{\theta}}_0})({{\underline{\theta}}_1} - {{\underline{\theta}}_0})}{d}.
\end{align}
In addition, t is a scalar, ${J}({{\underline{\theta}}_0})$ is $d\times d$ matrix, and $\left({{\underline{\theta}}_1} -{{\underline{\theta}}_0}\right)$ is a $d-dimensional$ vector, which is the element-wise subtraction of two d-dimensional vectors ${\underline{\theta}}_1$ and ${\underline{\theta}}_0$.

\end{proposition}
% Compared to Lemma \ref{thm:target_only} and Theorem \ref{thm:one_source}, Proposition \ref{Proposition:target_only} and Proposition \ref{Proposition:one_source} 
Compared to Theorem \ref{thm:one_source}, Proposition \ref{Proposition:one_source} exhibits a similar mathematical form, allowing us to derive the optimal transfer quantity through a similar approach.  
Furthermore, we observe that as the parameter dimension $d$ increases, the K-L error measure \eqref{Proposition:one_source_eq} increases. This suggests that for a complex model, knowledge transfer across tasks becomes more challenging, which is consistent with the findings in \cite{tong2021mathematical}.

\subsection{Multi-Source Transfer Learning}
Consider the multi-source transfer learning scenario with $K$ source task $\left\{\cS_1,\dots,\cS_K\right\}$ and one target task $\cT$. 
We aim to derive the optimal transfer quantity $n_i^*$ of each source.

\begin{theorem}(proved in Appendix \ref{appendix:multi_source})
\label{thm:multi_source}
Given a target task $\cT$ with $N_0$ i.i.d. samples generated from the underlying model $P_{X;{\underline{\theta}}_0}$, and K source tasks $\cS_1,\dots,\cS_K$ with $N_1,\dots,N_K$ i.i.d. samples generated from the underlying model $P_{X;{\underline{\theta}}_1},\dots,P_{X;{\underline{\theta}}_K}$, where ${\underline{\theta}}_0,{\underline{\theta}}_1,\dots,{\underline{\theta}}_K\in \mathbb{R}^d$. $\hat{{\underline{\theta}}}$ is denoted as the MLE \eqref{mle_defination} based on the $N_0$ samples from $\cT$ and $n_1,\dots,n_K$ samples from $\cS_1,\dots,\cS_K$, where $n_i \in [0,N_i]$. Denoting $s=\sum\limits_{i=1}^{K}n_i$ as the total transfer quantity, and $\alpha_i=\frac{n_i}{s}$ as the proportion of different source tasks, then the proposed measure \eqref{eq:KL1} can be expressed as: 
\begin{align}
\label{eq:multisourcetarget}
&\frac{d}{2}\left(\frac{1}{N_0+s}+\frac{s^2}{(N_0+s)^2}t\right)+o\left(\frac{1}{N_{0}+s}\right).
\end{align}
In  \eqref{eq:multisourcetarget}, $t$ is a scalar denoted as
\begin{align}
t=\frac{\underline{\alpha}^T\Theta^T{J}({{\underline{\theta}}_0})\Theta\underline{\alpha}}{d},
\end{align}
% $t=\frac{\underline{\alpha}^T\Theta^T{J}({{\underline{\theta}}_0})\Theta\underline{\alpha}}{d},$
where $\underline{\alpha}=\left[\alpha_1,\dots,\alpha_K\right]^T$ is a K-dimensional vector, and $\Theta^{d\times K}=\left[{{\underline{\theta}}_1}-{{\underline{\theta}}_0},\dots,{{\underline{\theta}}_K}-{{\underline{\theta}}_0} \right]$.
\end{theorem}

% We provide the method to obtain $s^*$ and $\alpha^*$ which minimize \eqref{eq:multisourcetarget} in the Appendix \ref{appendix:sa}.
According to Theorem \ref{thm:multi_source}, we can derive the optimal transfer quantities $n_1^*, \dots, n_K^*$ by minimizing \eqref{eq:multisourcetarget}. Equivalently, we need to find the optimal total transfer quantity $s^*$ and the optimal proportion vector $\underline{\alpha}^*$ which minimize \eqref{eq:multisourcetarget}. The analytical solutions of $s^*$ and $\underline{\alpha}^*$  are difficult to acquire, and we provide a method to get their numerical solutions in Appendix \ref{appendix:sa}. Eventually, we can get the optimal transfer quantity of each source through $n_i^*=s^*\cdot\alpha_i^*$ .

\subsection{Practical Algorithm}
% 值得注意的是，我们的方法还兼具了计算的高效性，在模型训练的同时可以根据其反向传播时的梯度进行fisher参数（即模型参数权重）的计算。由于我们的方法会用矩阵乘法将维度转化为很小的任务个数的维度，对于参数量比较大的模型而言，同一时刻也最多只会存在源任务模型的参数量，保证了我们方法对存储空间并没有太高要求。
% Along with our theoretical framework, we propose a practical algorithm, \ourmethod{}, which is applicable to all supervised target tasks, as presented  in Algorithm \ref{alg:ours_algorithm}. It mainly involves two stages: 
% (1) initializing the source task parameters $\theta_1,...,\theta_K$
% to compute the optimal transfer quantities
% % simulating the source data distribution of each source 
% and (2) jointly training the target model using a resampled dataset whose sample quantity of each source corresponds to the optimal transfer quantity derived from our theory. 

Along with our theoretical framework, we propose a practical algorithm, \ourmethod{}, which is applicable to sample-based multi-source transfer learning tasks, as presented in Algorithm \ref{alg:ours_algorithm}.
In \ourmethod{}, when computing the optimal transfer quantities based on Theorem \ref{thm:multi_source}, we use the parameters of pretrained source models to replace $\underline{\theta}_1,\cdots,\underline{\theta}_K$. Considering that the source model can be trained using sufficient labeled data, it is reasonable to use the learned parameters as a good approximation of the true underlying parameters.
In contrast, the number of target data in transfer learning is often insufficient, so it is difficult to accurately estimate the true parameter $\underline{\theta}_0$ - the parameter of the target task model - using only the target data. 
% In contrast, the quantity of target data in transfer learning is often limited, making it challenging to accurately estimate \( \underline{\theta}_0 \)—the target model parameter—using only the target data. As a result, it is also hard to obtain the \( \underline{\theta}_0 \) required for computing the optimal transfer quantities.
Therefore, as shown in lines 5-14 of Algorithm 1, we adopt a \textbf{dynamic strategy}. 
Specifically, in the first epoch, we train a $\underline{\theta}_0$ using only the target data. 
This $\underline{\theta}_0$ is then used, along with Theorem \ref{thm:multi_source}, to determine the optimal transfer quantity from each source task, and we use \textbf{random sampling} to form a new resampled training dataset. Finally, we continue training $\underline{\theta}_0$ on this new dataset, and this procedure is repeated in each subsequent epoch to iteratively update the training dataset. This mechanism helps alleviate the inaccuracy of $\underline{\theta}_0$, and we also validate the effectiveness of this design in Section \ref{dynamic_transfer}. In particular, we compute the matrix $J$ using the gradient of $\cL_{train}$ in line 10, and the loss function $\ell$ is the negative log-likelihood, following a widely adopted approach in deep learning known as the \textbf{empirical Fisher} \cite{martens2020new,osawa2023pipefisher}.

% Note that computing the KL-measure adds relatively little overhead to the SGD training as the Fisher information can be computed during training by leveraging the gradients obtained in the backpropagation process.

% It is noteworthy that the Fisher information can be computed during training by leveraging the gradients obtained in the backpropagation process.

% Moreover, it employs matrix multiplication reducing the dimensionalty of models to the counts of tasks, ensuring high efficiency on memory consumption.

\begin{algorithm}[H]
   \caption{\ourmethod{}: Training}
   \label{alg:ours_algorithm}
\begin{algorithmic}[1]
    \STATE {\bfseries Input:} Target data $D_{\cT}=\{(z_\cT^i, y_\cT^i)\}_{i=1}^{N_0}$, source data $\{D_{S_k}=\{(z_{S_k}^i, y_{S_k}^i)\}_{i=1}^{N_{k}}\}_{k=1}^K$, model type $f_{{\underline{\theta}}}$ and its parameters ${\underline{\theta}}_0$ for target task and $\{{\underline{\theta}}_{k}\}_{k=1}^K$ for source tasks, parameter dimension $d$.
    % task-specific loss function $\ell(y, \hat{y})$.
    \mycomment{$z$ represents the feature and $y$ represents the label}
    
    \STATE {\bfseries Parameter:} Learning rate $\eta$.
   
    \STATE {\bfseries Initialize:} randomly initialize ${\underline{\theta}}_0$, use parameters of pretrained source models to initialize $\{{\underline{\theta}}_{k}\}_{k=1}^K$.
   
    \STATE {\bfseries Output:} a well-trained $\underline{\theta}_0$ for target task model $f_{{\underline{\theta}}_0}$ .

    % \STATE {\bfseries Output:} a well-trained ${\underline{\theta}}_0$ for target task model $f_{{\underline{\theta}}_0}$.
   
    % \STATE \textbf{for} $k = 1$ {\bfseries to} $K$ \textbf{do} \mycomment{Initialize the source parameters}
    % \STATE \hspace{1em} \textbf{repeat}
    % \STATE \hspace{1em} \hspace{1em}$L_{k} \gets \frac{1}{N_k} \sum\limits_{i=1}^{N_k} \ell\left(y_{S_k}^i, f_{{\underline{\theta}}_k}(z_{S_k}^i)\right)$
    % \STATE \hspace{1em} \hspace{1em}${\underline{\theta}}_{k} \gets {\underline{\theta}}_{k} - \eta \nabla_{{\underline{\theta}}_{k}} L_{k}$
    % \STATE \hspace{1em} \textbf{until} ${\underline{\theta}}_{k}$ converges;
    % \STATE \textbf{end for}
   
    \STATE $D_{train} \gets D_\cT$  \mycomment{Initialize the trainning dataset by target task samples}
    %= \{(z^i, y^i)\}_{i=1}^{N_0} 
   
    \STATE \textbf{repeat} \mycomment{Use dynamic strategy to train the target task}
    
    \STATE \hspace{1em}$\cL_{train} \gets \frac{1}{|D_{train}|} \sum\limits_{(y^i,z^i)\in{D_{train}}} \ell\left(y^i, f_{{\underline{\theta}}_0}(z^i)\right)$
    \STATE \hspace{1em}${\underline{\theta}}_0 \gets {\underline{\theta}}_0 - \eta \nabla_{{\underline{\theta}}_0} \cL_{train}$
   
    \STATE \hspace{1em}$\Theta\gets\left[{{\underline{\theta}}_1}-{{\underline{\theta}}_0},\dots,{{\underline{\theta}}_K}-{{\underline{\theta}}_0} \right]^T$
    \STATE \hspace{1em}${J}({{\underline{\theta}}_0}) \gets ( \nabla_{{\underline{\theta}}_0} \cL_{train})( \nabla_{{\underline{\theta}}_0} \cL_{train})^T$ 
    %%%%%%%%%%%%%%老版本优化方法%%%%%%%%%%%%%%%%%%%%%
    % \STATE \hspace{1em}$\alpha^* \gets \argmin\limits_{\alpha} \alpha^T\Theta^T{J}({{\underline{\theta}}_0})\Theta\alpha $
    % \STATE \hspace{1em}$t^* \gets \frac{\alpha^{*^T}\Theta^T{J}({{\underline{\theta}}_0})\Theta\alpha^*}{d}$
    % \STATE \hspace{1em}$s^* \gets \argmin\limits_{s} \frac{d}{2}\left(\frac{1}{N_0+s}+\frac{s^2}{(N_0+s)^2}t^*\right)$ 
    %%%%%%%%%%%%%%老版本优化方法%%%%%%%%%%%%%%%%%%%%%
    \STATE \hspace{1em}$(s^*, \underline{\alpha}^*) \gets \argmin\limits_{(s, \underline{\alpha})} \frac{d}{2}\left(\frac{1}{N_0+s}+\frac{s^2}{(N_0+s)^2}\frac{\underline{\alpha}^T\Theta^T{J}({{\underline{\theta}}_0})\Theta\underline{\alpha}}{d}\right)$ 
    
    % \STATE \hspace{1em}$D_{source} \gets $for $k=1,\dots,K$, take $s^*\alpha_k^*$ samples from $D_{S_k}$ respectively.
    % \STATE \hspace{1em}
    % $D_{source} \gets \bigcup\limits_{k=1}^K \text{i.i.d. } D_{S_k}^{s^* \alpha_k^*}$
    % \STATE \hspace{1em}$D_{source} \gets \bigcup\limits_{k=1}^K  \left\{D_{S_k}^{*} \bigg| D_{S_k}^{*}\subseteq D_{S_k},\; \lvert D_{S_k}^{*}\rvert= s^* \alpha_k^*\right\}$
    %带rand的集合属于号↓
    \STATE \hspace{1em}$D_{source} \gets \bigcup\limits_{k=1}^K  \left\{D_{S_k}^{*} \bigg| D_{S_k}^{*}\overset{\text{rand}}{\subseteq} D_{S_k}, \lvert D_{S_k}^{*}\rvert= s^* \alpha_k^*\right\}$
    \STATE \hspace{1em}$D_{train} \gets D_{source}\bigcup D_\cT $\mycomment{Update the trainning dataset}
    \STATE \textbf{until} ${\underline{\theta}}_{0}$ converges;
\end{algorithmic}
\end{algorithm}

\section{Experiments}
\label{sec:experi}
% 总experiment
% 数据集、mertics
% 实现细节

% 大表结果分析 accuracy
% 数据使用量的分析（epoch等各种考虑哪个有优势画哪个）+计算时间分析 efficient 

% ablation analysis
% 模型取数据与domain之间距离的关系
% 不仅fewshot可以，在大的数据量场景下也能work
% LoRA也可以work

\subsection{Experiments Settings}
\label{Experiments_Settings}
\textbf{Benchmark Datasets.} 
\texttt{DomainNet} contains 586,575 samples of 345 classes from 6 domains (\textit{i.e.}, \textbf{C}: Clipart, \textbf{I}: Infograph, \textbf{P}: Painting, \textbf{Q}: Quickdraw, \textbf{R}: Real and \textbf{S}: Sketch). \texttt{Office-Home} benchmark contains 15588 samples of 65 classes, with 12 adaptation scenarios constructed from 4 domains:  Art, Clipart, Product and Real World (abbr. \textbf{Ar}, \textbf{Cl}, \textbf{Pr} and \textbf{Rw}). 
\texttt{Digits} contains four-digit sub-datasets: MNIST(mt), Synthetic(sy), SVHN(sv) and USPS(up), with each sub-dataset containing samples of numbers ranging from 0 to 9.

% \textbf{Implementation Details.} We employ the ViT-Small~\cite{rw2019timm_vits} model pre-trained on ImageNet21k~\cite{deng2009imagenet} as the backbone model on all datasets. We adopt Adam optimizer with $1e^{-5}$ learning rate, use 20\% of the dataset as the test set, and report the highest accuracies within 5 epoch early stops in all experiments. Following the standard protocol of few-shot learning, the training data for k-shot is k samples per class randomly selected from the target task. All the experiments are conducted on A800 GPUs.

\textbf{Implementation Details.} We employ the ViT-Small model \cite{rw2019timm_vits}, pre-trained on ImageNet-21k \cite{deng2009imagenet}, as the backbone for all datasets. The Adam optimizer is employed with a learning rate of $1e^{-5}$. We allocate 20\% of the dataset as the test set, and report the highest accuracies within 5 epoch early stops in all experiments. Following the standard few-shot learning protocol, the training data for k-shot consists of k randomly selected samples per class from the target task. All experiments are conducted on Nvidia A800 GPUs.

% The scope of compared methods includes: 
% 1) \textbf{Unsupervised Methods}: MSFDA\cite{shen2023balancingMSFDA}, DATE\cite{han2023discriminability_DATE}, M3SDA\cite{peng2019moment_M3SDA}.
% 2) \textbf{Few-Shot Methods with Parameter-Weighting}: H-ensemble\cite{wu2024h_Hensemble}, MCW\cite{lee2019learning_MCW}.
% 3) \textbf{Few-Shot Methods with Samples-Relation:} MADA\cite{zhang2024revisiting_MADA} , WADN\cite{shui2021aggregating_WADN}.
% 4) \textbf{Source Ablating}: Target-Only, Single-Source-Avg, Single-Source-Best, \allsource{}; 

% and considering few methods are designed for the few-shot data-efficient transfer setting

\textbf{Baselines.} For a general performance evaluation, we take SOTA works under similar settings as baselines.
% we unified the methods of existing works under similar settings as ours (with varied source data.) 
The scope of compared methods includes: 
1) Unsupervised Methods: MSFDA \cite{shen2023balancingMSFDA}, DATE \cite{han2023discriminability_DATE}, M3SDA \cite{peng2019moment_M3SDA}.
2) Few-Shot Methods Based on Model(Parameter)-Weighting: H-ensemble \cite{wu2024h_Hensemble}, MCW \cite{lee2019learning_MCW}.
3) Few-Shot Methods Based on Sample: MADA \cite{zhang2024revisiting_MADA},  WADN~\cite{shui2021aggregating_WADN} 
4) Source Ablating Methods: Target-Only, Single-Source-Avg/Single-Source-Best (average/best performance of single-source transfer), \allsource{} (all source \& target data in multi-source transfer).  

%The ResNet\cite{he2016deep_resnet} backbones they used are all pretrained on ImageNet21k. 
%Here, ``Single-$*$'' attaches one domain dataset from source to the target one for training together and evaluating on target testset while ‘‘\allsource{}’’ attaches all domain datasets, and $*$-Best/Avg stand for the best/average performance. Target-Only finetunes Imagenet-pretrained model on few-shot target data only. 
Note that MADA~\cite{zhang2024revisiting_MADA} leverages all unlabeled target data in conjunction with a limited amount of labeled target data, which is a hybrid approach combining unsupervised and supervised learning. 
%, which means all the target samples have been learned to some extent.
Due to the page limit, we provide detailed information on the experimental design and the results of an experiment adapted to the WADN settings on the \texttt{Digits} dataset in Appendix \ref{appendix:Experimental_Design_and_Model Adaptation}.
\begin{table*}[!htbp]
\centering
\caption{\textbf{Multi-Source Transfer Performance on DomainNet and Office-Home.} The arrows indicate transfering from the rest tasks. The highest/second-highest accuracy is marked in Bold/Underscore form respectively. }
\resizebox{\textwidth}{!}{
\begin{tabular}{lc c ccccccc c ccccc}
\toprule
% \multirow{2}{*}{\makecell{\textbf{Method} \\ Second line}} &\multirow{2}{*}{\makecell{\textbf{Method} \\ Second line}} 
\multirow{2.5}{*}{\textbf{Method}} & \multirow{2.5}{*}{\textbf{Backbone}} && \multicolumn{7}{c}{\textbf{DomainNet}} && \multicolumn{5}{c}{\textbf{Office-Home}}\\ 
% \cline{3-9} \cline{10-14}
\cmidrule(lr){4-10} \cmidrule(lr){12-16}
 &&&$\to$C & $\to$I & $\to$P & $\to$Q & $\to$R & $\to$S & Avg&& $\to$Ar & $\to$Cl & $\to$Pr & $\to$Rw & Avg\\ 

 % \multicolumn{1}{c|}{\makecell{\textbf{Method}}} &\multicolumn{1}{c|}{\makecell{\textbf{Backbone}}} &\multicolumn{7}{c|}{\makecell{\textbf{DomainNet}\\\midrule$\to$C\hspace{1.4em}$\to$I\hspace{1.4em}$\to$P\hspace{1.4em}$\to$Q\hspace{1.4em}$\to$R\hspace{1.4em}$\to$S\hspace{1.4em}Avg}} &\multicolumn{5}{c}{\makecell{\textbf{Office-Home}\\\midrule$\to$Ar\hspace{0.9em}$\to$Cl\hspace{0.9em}$\to$Pr\hspace{0.9em}$\to$Rw\hspace{0.9em}Avg}}\\
\midrule
% \midrule
% \textit{Source-combine:} &&&&&&&&&&&&\\
% \textit{Unsupervised-all-shots:} &&&&&&&&&&&&&\\
\multicolumn{14}{l}{\textbf{\textrm{Unsupervised-all-shots}}} \\
% (Surrogate-RN50-ICML23)
MSFDA\cite{shen2023balancingMSFDA} & ResNet50 && 66.5 & 21.6 &56.7 &20.4 &70.5 &54.4 &48.4  && 75.6 &62.8 &84.8 &\underline{85.3} &77.1  \\
% (SHOT-RN50-AAAI23)
DATE\cite{han2023discriminability_DATE}& ResNet50 && - & - & - & - & - & -   & - && 75.2 & 60.9 &  \textbf{85.2} & 84.0 & 76.3 \\
% MSFDA(Oracle-RN50-ICML23): & 76.5 &32.8 &64.7 &34.6 &77.8 &62.4 &58.1  & 86.9 &85.1 &92.2 &95.7 &90.0  \\
% M3SDA(*-RN101-CVPR19)& 57.0 & 22.1 & 50.5 & 4.4 & 62.0 & 48.5 & 40.8  & - & - & - & - & - \\
% (RN101-CVPR19)
M3SDA\cite{peng2019moment_M3SDA}& ResNet101 &&57.2 & 24.2 & 51.6 & 5.2 & 61.6& 49.6& 41.5  && - & - & - & - & - \\
% M3SDA($\beta$-RN101-CVPR19)& 58.6 & 26.0 & 52.3 & 6.3 & 62.7&  49.5 & 42.6  & - & - & - & -  & - \\
\midrule
% \textit{Supervised-10-shots:} &&&&&&&&&&&&&\\
\multicolumn{14}{l}{\textbf{\textrm{Supervised-10-shots}}} \\
\multicolumn{14}{l}{\textit{\textbf{Few-Shot Methods:}}}\\
% H-ensemble(Vits-AAAI24)& 51.7 & 21.7 & 54.0 & 19.2 & 70.5 & 42.0 & 43.2 & - & - & - & - & -  \\
% (Vits-AAAI24)
H-ensemble\cite{wu2024h_Hensemble}& ViT-S && 53.4 & 21.3 & 54.4 & 19.0 & 70.4 & 44.0 & 43.8 && 71.8 & 47.5 & 77.6 & 79.1 & 69.0  \\
% (Vits-CVPR24)
% MADA\cite{zhang2024revisiting_MADA}& ViT-S && 51.0 & 11.0 & 60.3 & 15.0 & \textbf{81.4} & 16.1 & 39.1 && \textbf{81.7} & 39.7 & \textbf{86.9} & \textbf{91.3} & 74.9 \\
% MADA\cite{zhang2024revisiting_MADA}& ViT-S && 51.0 & 11.0 & 60.3 & 15.0 & \textbf{81.4} & 16.1 & 39.1 && \textbf{85.2} & \textbf{73.1} & \textbf{93.9} & \textbf{92.7} & \textbf{86.2} \\
MADA\cite{zhang2024revisiting_MADA}& ViT-S && 51.0 & 12.8 & 60.3 & 15.0 & \textbf{81.4} & 22.7 & 40.5 && \textbf{78.4} & 58.3 & 82.3 & 85.2 & 76.1 \\
% (Res50-CVPR24)
% MADA\cite{zhang2024revisiting_MADA}& Resnet50 && 66.1 & 23.9 & \underline{60.4} & \underline{31.9} & \underline{75.4} & 52.5 & 51.7 && 72.8 & 59.9 & 82.4 & 81.5 & 74.2 \\
% MADA\cite{zhang2024revisiting_MADA}& Resnet50 && 66.1 & 23.9 & \underline{60.4} & \underline{31.9} & \underline{75.4} & 52.5 & 51.7 && 76.5 & 66.9 & 84.6 & 84.1 & 78.0 \\
% MADA\cite{zhang2024revisiting_MADA}& Resnet50 && 66.1 & 23.9 & \underline{60.4} & \underline{31.9} & \underline{75.4} & 52.5 & 51.7 && 72.2 & 64.5 & 82.9 & 81.9 & 75.4 \\
% MADA\cite{zhang2024revisiting_MADA}& ResNet50 && 66.1 & 23.9 & \underline{60.4} & \underline{31.9} & \underline{75.4} & 52.5 & 51.7 && \underline{79.6} & \underline{70.3} & \underline{86.2} & 84.6 & \underline{80.2} \\
MADA\cite{zhang2024revisiting_MADA}& ResNet50 && 66.1 & 23.9 & \underline{60.4} & \underline{31.9} & \underline{75.4} & 52.5 & 51.7 && 72.2 & \underline{64.4} & 82.9 & 81.9 & 75.4 \\
% MCW(Vits-NeurIPS19) & 55.3 & 20.5 & 52.8 & 20.4 & 70.0 & 42.7 & 43.6 & - & - & - & - & - \\
% (Vits-NeurIPS19)
MCW\cite{lee2019learning_MCW}& ViT-S && 54.9 & 21.0 & 53.6 & 20.4 & 70.8 & 42.4 & 43.9 && 68.9 & 48.0 & 77.4 & \textbf{86.0} & 70.1 \\
% (Vits-ICML21)
WADN\cite{shui2021aggregating_WADN}& ViT-S && 68.0 & 29.7 & 59.1 & 16.8 & 74.2 & 55.1 & 50.5 && 60.3 & 39.7 & 66.2 & 68.7 & 58.7 \\
% \midrule
% \textit{Supervised-10-shots:} &&&&&&&&&&&&&\\
\multicolumn{14}{l}{\textbf{\textit{Source-Ablation Methods:}}} \\
Target-Only& ViT-S && 14.2 & 3.3  & 23.2  & 7.2  & 41.4   & 10.6  & 16.7 && 40.0 & 33.3 & 54.9 & 52.6 & 45.2 \\ 
% Single-Source-best(Vits) & - & - & -  & -  & - & -  & - & -  & - & - & - & - \\ 
% Single-Source-avg(Vits) & - & - & -  & -  & - & -  & - & -  & - & - & - & - \\

Single-Source-Avg& ViT-S && 50.4 & 22.1 & 44.9  & 24.7  & 58.8 & 42.5  & 40.6 && 65.2  & 53.3 & 74.4 & 72.7 & 66.4 \\
Single-Source-Best& ViT-S && 60.2 & 28.0 & 55.4  & 28.4  & 66.0 & 49.7  & 48.0 && 72.9  & 60.9 & 80.7 & 74.8 & 72.3 \\ 

\allsource{}& ViT-S && \underline{71.7}   & \underline{32.4}  & 60.0 & 31.4 & 71.7 & \underline{58.5} & \underline{54.3} && 77.0 & 62.3 & 84.9 & 84.5 & \underline{77.2}\\
\ourmethod{} (Ours)& ViT-S && \textbf{72.8} & \textbf{33.8} & \textbf{61.2} & \textbf{33.8}  & 73.2 & \textbf{59.8} & \textbf{55.8} && \underline{78.1} & \textbf{64.5} & \textbf{85.2} & 84.9  & \textbf{78.2} \\
\bottomrule
\end{tabular}
}
\label{tab:major}
\end{table*}

\subsection{Main Result}
%\subsection{Main Result: Quantity with Quality Over Only Quantity}
%\subsection{Main Result on Classification Tasks}
% 大表结果分析 accuracy
% 可分多个baseline针对性分析
% 这里的conditions要补！！
We evaluated our algorithm, \ourmethod{}, alongside baseline methods on the few-shot multi-source transfer learning tasks using the \texttt{DomainNet} and \texttt{Office-Home} datasets. The quantitative results are summarized in Table \ref{tab:major}. Since the unsupervised baselines are not designed for the supervised few-shot setting, we report their original results from the respective papers for reference.
We make the following observations:

\textbf{Overall Performance.} In general, compared to baseline methods, \ourmethod{} achieves the best performance on almost all the transfer scenarios on the two datasets. Specifically, \ourmethod{} outperforms the state-of-the-art (\allsource{}) by an average of $1.5\%$ on \texttt{DomainNet} and $1.0\%$ on \texttt{Office-Home}.

% aiming at finding a better model structure that can translate the source knowledge into target worth knowledge, tried complicated ways and just can not beat easily training the simple model using all the available samples.
% 2.对比supervised情况：大部分的supervised方法都是从模型的参数入手进行可迁移性的度量。在一样的数据情况下，可以看到仅仅从模型的参数入手，supervised的方法的表现并没有从数据的角度入手用all source数据集去训目标任务模型的效果好。

\textbf{Data Speaks Louder Than Model Weighting.} It is worth noting that on both datasets, sample-based methods utilizing both target and source samples to jointly train the model, such as WADN, MADA and \ourmethod{},
generally outperform model(parameter)-weighting approaches which construct the target model by weighting source models, such as H-ensemble and MCW. This observation suggests that sample-based approaches offer greater advantages over model-based methods, because they can fully leverage the relevant information from the source data for the target task.
%, except when constrained by source-free limitations.

% Using all the source samples and the same limited target samples, methods such as H-ensemble, MADA, MCW and WADN, attempts to design complex model structures to better transfer source knowledge to the target domain, but their sophisticated approaches fail to surpass the straightforward training of a simple model using all available data from the results about \allsource{} method on Table \ref{tab:major}.

% \textbf{Take, but only as much as you need.} At the Source-Ablation part of Table~\ref{tab:major} where we compare our results among Target-only, \allsource{}, and \ourmethod{}, we could observe that \ourmethod{} performed best in both datasets. This result confirmed our theory and demonstrated that benefits of searching for the sweet spot between the bias-variance trade-off. By choosing the right quantities of data samples from the source tasks, we could train the target model more accurately and efficiently. 

\textbf{Take, But Only as Much as You Need.} 
%At the Source-Ablation part of Table~\ref{tab:major} where we 
Comparing results in Table~\ref{tab:major} among Target-only, \allsource{}, and \ourmethod{}, we observe that \ourmethod{} achieves the best performance in both datasets by leveraging only a subset of data selected from all available sources based on model preference. This result validates our theory.
% and highlights the benefits of identifying the optimal balance in the bias-variance trade-off. 
By choosing the right quantities of samples from the source tasks, we could train the target model more accurately, and we give an analysis on the domain preference of transfer quantity in Appendix \ref{Analysis_Quantity}. Furthermore, Figure \ref{fig:train_efficiency} shows that \ourmethod{} also significantly reduces the training time and sample usage, which validates its superiority in terms of data efficiency.

\textbf{Few-Shot Labels, Big Gains.} We make a comparison of the results of unsupervised and supervised methods. While other conditions remain the same, 
%except for the number of target samples used, 
Table \ref{tab:major} demonstrates that even if unsupervised methods like MSFDA and M3SDA take all the target data into account (up to 1.3×$10^5$ samples on Real domain of \texttt{DomainNet}), their performance still falls short compared to the supervised methods, which rely on only a limited number of samples (3450 samples). This illustrates the importance of having supervised information in multi-source transfer learning.
% 1.对比unsupervised方法：它们的shot需求太高。它们的方法使用了all-shots，虽然是没有label的目标任务数据，但是使用了targetdomain的所有数据（带入相比最夸张的domain的使用数据量），而在supervised的all的方法只使用了每类10shot的情况（目标任务使用数据量情况）只用到了少量的带有标签的shot，效果就已经超越了大部分unsupervised的情况。
% 需不需要只是无监督来个fewshot然后情况很糟

\begin{wraptable}[8]{r}{0.45\textwidth}%
\vspace{-1.25em}
\centering
\caption{Static vs. Dynamic Transfer Quantity in \ourmethod{} on Office-Home.}
% \caption{\textbf{Dynamic \ourmethod{} Bias Ablation on Office-Home.} The arrows indicate transfering from the rest tasks.}
\resizebox{0.4\textwidth}{!}{
\begin{tabular}{l c ccccc}
\toprule
\multirow{2.5}{*}{\textbf{Method}} & \multirow{2.5}{*}{\textbf{Backbone}} & 
% \multirow{2.5}{*}{\textbf{All samples}} & 
\multicolumn{5}{c}{\textbf{Office-Home}}\\ 
\cmidrule(lr){3-7} 
&& $\to$Ar & $\to$Cl & $\to$Pr & $\to$Rw & Avg\\ 
\midrule
% \midrule
% \textit{Source-combine:} &&&&&&&&&&&&\\
% \textit{Supervised-10-shots Source-Ablation:} &&&&&\\
% None-Source(Vitb)  & 59.8 & 42.2 & 69.5 & 72.0 & 60.9 \\ 
% Single-Source-avg(Vitb)& 72.2  & 59.9 & 82.6 & 81.0 & 73.9 \\
% Single-Source-best(Vitb) & 74.4  & 61.8 & 84.9 & 81.9 & 75.8 \\ 
% \multicolumn{7}{l}{\textit{\textbf{Following settings of WADN:}}} \\
% 10604 9045 8991 9054
\multicolumn{7}{l}{\textit{\textbf{Supervised-10-shots:}}} \\
Static-Under & ViT-S & \underline{77.0} & \underline{62.3} & 84.9 & \underline{84.5} & \underline{77.2}\\
% 486 7139 8290 7948(0.79)
Static-Exact & ViT-S & 46.0 & 59.8  & \underline{85.1}  & 83.7  & 68.7\\
% 3730 
% 9792 5474 2351 2125
Static-Over & ViT-S & 76.8 & 61.9 & 78.6 & 68.6 & 71.5\\
% 3532
Dynamic & ViT-S & \textbf{78.1} & \textbf{64.5} & \textbf{85.2} & \textbf{84.9}  & \textbf{78.2} \\
\bottomrule
\end{tabular}
}
\label{tab:static_and_dynamic}
\vspace{-2mm} 
\end{wraptable}

\subsection{Static vs. Dynamic Transfer Quantity} 
\label{dynamic_transfer}
In our proposed Algorithm \ref{alg:ours_algorithm}, we employ a ‘‘Dynamic’’ strategy that dynamically determines the optimal transfer quantities and updates the resampled dataset during the joint training of target task. To validate the effectiveness of this strategy, 
we conducted comparative experiments using the ‘‘Static-\(*\)’’ methods. ‘‘Static-$*$’’ methods first simulate the distribution of target on target dataset only, and different types of Static such as ‘‘Under, Exact and Over’’ stands for different fitting levels. In ‘‘Static-\(*\)’’ methods, we only compute the optimal transfer quantity once to make the resampled dataset, and evaluated on it until target model converges. The results on Table \ref{tab:static_and_dynamic} demonstrate \ourmethod{} using dynamic transfer quantity achieved the best performance.

% \begin{table}[htbp]
% \centering
% \caption{Static vs. Dynamic Transfer Quantity in \ourmethod{} on Office-Home. Arrows will transfer from the rest tasks.}
% % \caption{\textbf{Dynamic \ourmethod{} Bias Ablation on Office-Home.} The arrows indicate transfering from the rest tasks.}
% \resizebox{0.4\textwidth}{!}{
% \begin{tabular}{l c ccccc}
% \toprule
% \multirow{2.5}{*}{\textbf{Method}} & \multirow{2.5}{*}{\textbf{Backbone}} & 
% % \multirow{2.5}{*}{\textbf{All samples}} & 
% \multicolumn{5}{c}{\textbf{Office-Home}}\\ 
% \cmidrule(lr){3-7} 
% && $\to$Ar & $\to$Cl & $\to$Pr & $\to$Rw & Avg\\ 
% \midrule
% % \midrule
% % \textit{Source-combine:} &&&&&&&&&&&&\\
% % \textit{Supervised-10-shots Source-Ablation:} &&&&&\\
% % None-Source(Vitb)  & 59.8 & 42.2 & 69.5 & 72.0 & 60.9 \\ 
% % Single-Source-avg(Vitb)& 72.2  & 59.9 & 82.6 & 81.0 & 73.9 \\
% % Single-Source-best(Vitb) & 74.4  & 61.8 & 84.9 & 81.9 & 75.8 \\ 
% % \multicolumn{7}{l}{\textit{\textbf{Following settings of WADN:}}} \\
% % 10604 9045 8991 9054
% \multicolumn{7}{l}{\textit{\textbf{Supervised-10-shots:}}} \\
% Static-Under & ViT-S & \underline{77.0} & \underline{62.3} & 84.9 & \underline{84.5} & \underline{77.2}\\
% % 486 7139 8290 7948(0.79)
% Static-Exact & ViT-S & 46.0 & 59.8  & \underline{85.1}  & 83.7  & 68.7\\
% % 3730 
% % 9792 5474 2351 2125
% Static-Over & ViT-S & 76.8 & 61.9 & 78.6 & 68.6 & 71.5\\
% % 3532
% Dynamic & ViT-S & \textbf{78.1} & \textbf{64.5} & \textbf{85.2} & \textbf{84.9}  & \textbf{78.2} \\
% \bottomrule
% \end{tabular}
% }
% \label{tab:static_and_dynamic}
% \vspace{-3mm} 
% \end{table}

\subsection{Generality across Different Shot Settings}
%由上述的理论分析可知，我们的方法并没有受限于目标任务的使用数据量，我们紧接着还继续探究了在目标任务的shot继续增加时我们的方法的表现情况，通过上图可以发现我们的方法在目标任务shot继续增加的情况下还可以稳步比其他方法效果更好，体现了我们方法在数据使用量上的的generative，更多的情况看附录。
As discussed in the theoretical analysis of Theorem~\ref{thm:one_source}, our theoretical framework is applicable to any quantity of target samples. Therefore, \ourmethod{} exhibits shot generality, enabling it to avoid negative transfer across different shot settings.
To validate this, we increase the number of shots from 5 to 100 across methods including \allsource{}, Target-Only, and \ourmethod{}. 
As shown in Figure \ref{fig:all_increase_shots}, experimental results demonstrate that \ourmethod{} consistently outperforms other approaches across all shot settings. This highlights the generality and scalability of \ourmethod{} in terms of data utilization. %Further results and details are provided in the appendix.
%%% real_figure_1_16 !!!!
% \begin{figure}[htbp]
%     \centering
%     % 第一排的两个子图
%     \begin{minipage}[t]{0.45\linewidth} % 子图宽度为 45%
%         \centering
%         \includegraphics[width=0.8\linewidth]{hgb_1_30_figure/Idomain_1_30.pdf}\\ % 调整图片宽度
%         \textbf{(a)} Infograph % 手动添加子图标题
%     \end{minipage}
%     % \hfill
%     \hspace{-0.5cm}
%     \begin{minipage}[t]{0.45\linewidth}
%         \centering
%         \includegraphics[width=0.8\linewidth]{hgb_1_30_figure/Pdomain_1_30.pdf}\\
%         \textbf{(b)} Painting
%     \end{minipage}
%     \vspace{0.5em} % 第一排与第二排之间的间距

%     % 第二排的两个子图
%     \begin{minipage}[t]{0.45\linewidth}
%         \centering
%         \includegraphics[width=0.8\linewidth]{hgb_1_30_figure/Qdomain_1_30.pdf}\\
%         \textbf{(c)} Quickdraw
%     \end{minipage}
%     \hspace{-0.59cm}
%     \begin{minipage}[t]{0.45\linewidth}
%         \centering
%         \includegraphics[width=0.8\linewidth]{hgb_1_30_figure/Rdomain_1_30.pdf}\\
%         \textbf{(d)} Real
%     \end{minipage}

%     % \caption{Effectiveness with Increasing Target Data}
%     \caption{Performance comparison with increasing target shots up to 100 per class on DomainNet dataset (I, P, Q and R domains). \ourmethod{} (\textcolor{blue}{blue}) outperforms other methods.}
%     \label{fig:all_increase_shots}
% \end{figure}

\begin{figure}[htbp]
    \centering
    \includegraphics[width=1.00\linewidth]{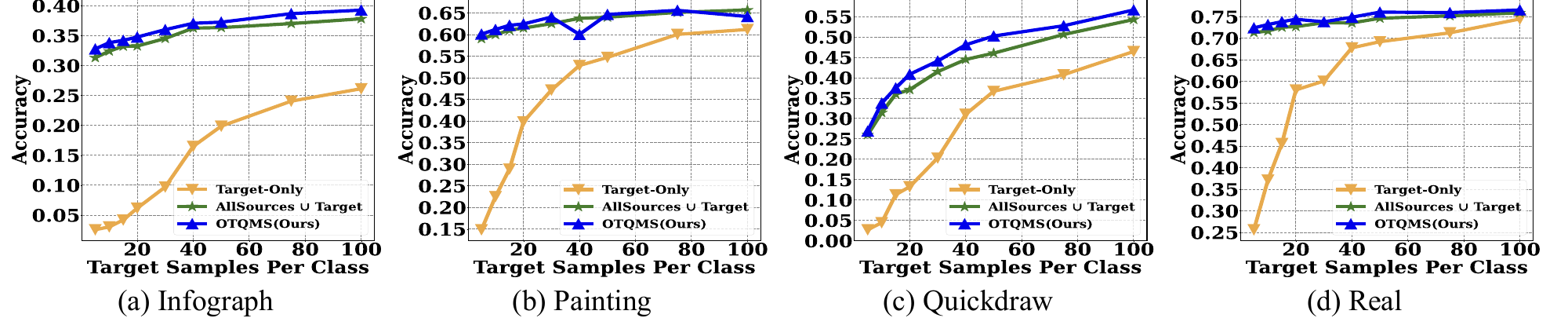}
    \caption{Performance comparison with increasing target shots up to 100 per class on DomainNet dataset (I, P, Q and R domains). \ourmethod{} (\textcolor{blue}{blue}) outperforms other methods.}
    \label{fig:all_increase_shots}
    % \vspace{-3cm}
\end{figure}

% And from figure[], our experiments of evaluating the effectiveness of our method on the large shots have proved that our method continues to outperform other approaches steadily with target shots increasing. This demonstrates the generative capability of our method in terms of data utilization. Additional results and details can be found in the appendix.
% To further investigate this, we conducted experiments to evaluate the performance of our method as the number of shots in the target task increases. As shown in the figure, our method continues to outperform other approaches steadily, even as the number of target task shots increases. This demonstrates the generative capability of our method in terms of data utilization. Additional results and details can be found in the appendix.
\subsection{Data Efficiency Test} 
\label{section:data_efficient}
In this section, we demonstrate the advantage of \ourmethod{} in terms of data efficiency.
Specifically, we conduct experiments with MADA, \allsource{}, and \ourmethod{} across different shot settings, and for each shot setting, we accumulate the total sample used and time consumed until the highest accuracy is reached.
% We visualize the results of the average of all shot setting
% To better visualize the results, we show the average data used and time consumed of all shot setting on Figure \ref{fig:train_efficiency}, which demonstrates that \ourmethod{} requires significantly less data and training time while maintaining competitive performance. 
To better visualize the results, 
we present the average sample usage and training time across all shot settings in Figure~\ref{fig:train_efficiency}. %which demonstrates that \ourmethod{} requires significantly less data and training time while maintaining competitive performance.
To be specific, \ourmethod{} reduces the average training time by $35.19\%$ and the average sample usage by $47.85\%$ on \texttt{DomainNet}, compared to the \allsource{} method.

% Specifically, we accumulate all the data used and time consumed until the highest accuracy is reached for each shot in \allsource{} and \ourmethod{} strategies. 
% To better visualize the results we show the average of all shots on Figure \ref{fig:train_efficiency}, which demonstrates that \ourmethod{} requires significantly less data and training time while maintaining competitive performance.

% In this section, we demonstrate the advantage of \ourmethod{} in terms of data efficiency.
% Specifically, we accumulate all the data used and time consumed until the highest accuracy is reached for each shot in \allsource{} and \ourmethod{} strategies. To better visualize the results we show the average of all shots on Figure \ref{fig:train_efficiency}, which demonstrates that \ourmethod{} requires significantly less data and training time while maintaining competitive performance.  
% and more details for each shot are provided in the appendix.
% \definecolor{MADAvits}{cmyk}{0.918, 0, 0}
% \definecolor{MADAres50}{cmyk}{0, 0.918 ,0.918}
% \definecolor{allsources}{cmyk}{0.290, 0.486, 0.192}
% \definecolor{ours}{cmyk}{0, 0, 0.918}
\begin{figure}[htbp]
    \centering
    \includegraphics[width=0.8\linewidth]{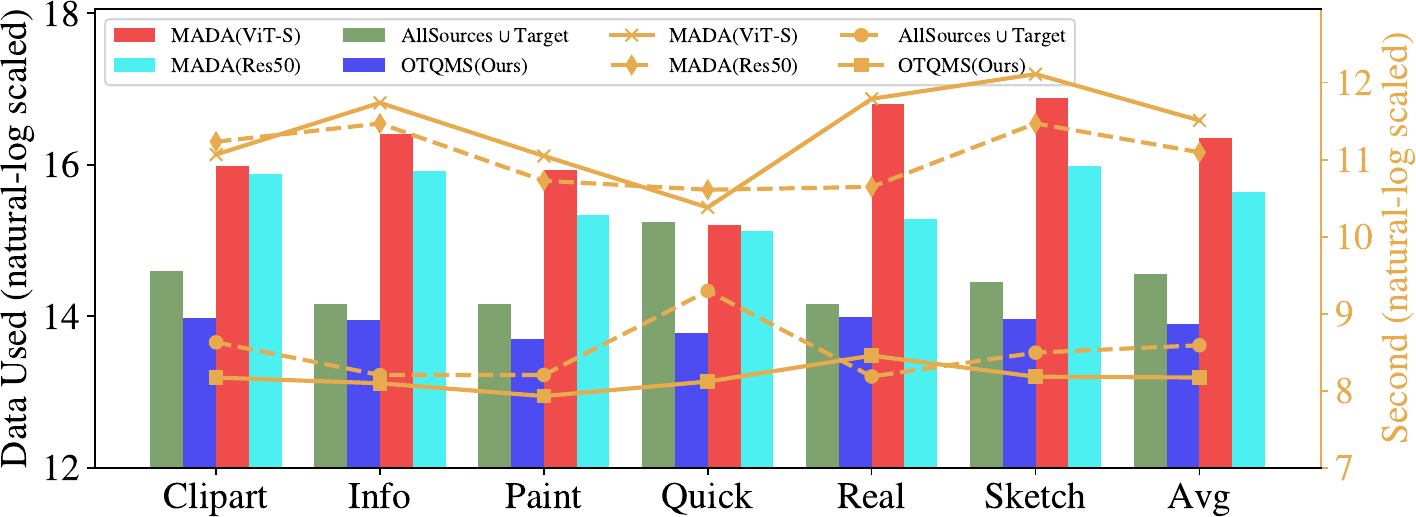}
    \caption{Data efficiency comparison of average sample usage and training time on \texttt{DomainNet} dataset, the left vertical axis represents the amount of sample usage, with \textcolor{allsources}{green} bars indicating \allsource{} data counts, \textcolor{ours}{blue} bars about \ourmethod{}, \textcolor{MADAvits}{red} bars about MADA(ViT-S) and \textcolor{MADAres50}{azury} bars about MADA(Res50), while the right \textcolor{orange}{orange} vertical axis and lines represent training time. }
    % All the results in the figure are scaled by the natural logarithm.
    % , shown by the dashed yellow line for \allsource{} and the solid orange line for \ourmethod{}.}
    \label{fig:train_efficiency}
\end{figure}

\subsection{
Compatibility to Parameter-Efficient Training}
% partial parameters finetune efficient
% Applicability to Larger Model Architectures
% LoRA（Low-Rank Adaptation）是一种针对大模型的低秩适应方法，它冻结预训练模型权重，并在 Transformer 架构的每一层注入可训练的秩分解矩阵，大大减少了下游任务的可训练参数数量，在保证模型质量的同时提高了训练效率且不增加推理延迟。为了测试我们方法在大模型下的适用性和高效性，我们使用添加了LoRA框架的vit-b模型作为新的backbone，与上述一样的setting在OfficeHome数据集上进行测试，结果如下表所示可以看出我们的方法在大模型情况下也有效。

% To evaluate the applicability with larger models, we used a ViT-base model (three times larger than ViT-small) integrated with the LoRA framework as the backbone. 

To evaluate the applicability with parameter-efficient training, we used a ViT-Base model~\cite{dosovitskiy2020vit_vitb} integrated with the LoRA framework \cite{hu2021lora} as the backbone. 
%LoRA \cite{hu2021lora} is a framework designed for large models, which freezes the original model weights and injects trainable low-rank decomposition matrices into each Transformer architecture. 
This approach significantly reduces the number of trainable parameters for downstream tasks while ensuring model quality and improving training efficiency. 
In the experiments, we consider the trainable low-rank matrices and the classification head as the parametric model in our theoretical framework, while treating the remaining parameters as constants. Other experimental settings are the same as default.
Experiments were conducted on the \texttt{Office-Home} dataset.
%under the same settings as details described above. 
The results of Table~\ref{tab:LoRA} in Appendix demonstrate that our method remains effective.
% even with larger model architectures.
% 0.5978
% 0.4220
% 0.6953
% 0.7199
% 0.8108
% 0.6596
% 0.8802
% 0.8924
% 0.8151
% 0.6797
% 0.8919
% 0.9031

% \subsection{Tendency Towards Parameter-Similar Model}
% \subsection{Observation of the domains poverties}
\begin{wraptable}[7]{r}{0.45\textwidth}%
\vspace{-1.25em}
\centering
\caption{Multi-task performance on four tasks of Office-Home.}
% \caption{\textbf{Dynamic \ourmethod{} Bias Ablation on Office-Home.} The arrows indicate transfering from the rest tasks.}
\resizebox{0.4\textwidth}{!}{
\begin{tabular}{l c ccccc}
\toprule
\multirow{2.5}{*}{\textbf{Method}} & \multirow{2.5}{*}{\textbf{Backbone}} & 
% \multirow{2.5}{*}{\textbf{All samples}} & 
\multicolumn{5}{c}{\textbf{Office-Home}}\\ 
\cmidrule(lr){3-7} 
&&Ar & Cl & Pr & Rw & Avg\\ 
\midrule
% \midrule
% \textit{Source-combine:} &&&&&&&&&&&&\\
% \textit{Supervised-10-shots Source-Ablation:} &&&&&\\
% None-Source(Vitb)  & 59.8 & 42.2 & 69.5 & 72.0 & 60.9 \\ 
% Single-Source-avg(Vitb)& 72.2  & 59.9 & 82.6 & 81.0 & 73.9 \\
% Single-Source-best(Vitb) & 74.4  & 61.8 & 84.9 & 81.9 & 75.8 \\ 
% \multicolumn{7}{l}{\textit{\textbf{Following settings of WADN:}}} \\
% 10604 9045 8991 9054
% 486 7139 8290 7948(0.79)
% 3730 
% 9792 5474 2351 2125
Single-task & ViT-S & 66.7 & 62.3 & 87.8 & 68.6 & 71.4\\
% 3532
\ourmethod{} & ViT-S & \textbf{81.7} & \textbf{76.0} & \textbf{88.6} & \textbf{87.5}  & \textbf{83.5} \\
\bottomrule
\end{tabular}
}
\label{tab:Multi_task_exp}
\vspace{-2mm} 
\end{wraptable}

\subsection{Compatibility to Multi-Task Learning}
To demonstrate the applicability of our method to multi-task learning scenarios, we conduct experiments on the \texttt{Office-Home} dataset using the ViT-Small model. In multi-task learning, each task simultaneously serves as both a source and a target task. In the experiments, each task is treated as the target task in \ourmethod{} in turn during training, while the transfer quantities from all source tasks are computed in turn. This setup enables us to evaluate how effectively \ourmethod{} leverages information across tasks. We compare the performance of our method against single-task training to evaluate its effectiveness in Table \ref{tab:Multi_task_exp}.

\section{Conclusion}

In this work, we propose a theoretical framework to determine
the optimal transfer quantities in multi-source transfer learning. 
Our framework reveals that by optimizing the transfer quantity of each source task, we can improve target task training while reducing the total transfer quantity. 
Based on this theoretical framework, we develop an architecture-agnostic
and data-efficient practical algorithm \ourmethod{} for jointly training the target model. 
We evaluated the proposed algorithm through extensive experiments and demonstrated its superior accuracy and enhanced data efficiency.
%%%, highlighting the non-uniform and non-symmetric relationships between tasks

\section*{Acknowledge}
The research is supported in part by 
National Key R\&D Program of China under Grant 2021YFA0715202, the National Natural Science Foundation of China under Grants 62571296, the Shenzhen Science and Technology Program under KJZD20240903102700001, and the Natural Science Foundation of China (Grant 62371270).

\bibliography{example_paper}
\bibliographystyle{plainnat}

%%%%%%%%%%%%%%%%%%%%%%%%%%%%%%%%%%%%%%%%%%%%%%%%%%%%%%%%%%%%

\newpage

\section*{NeurIPS Paper Checklist}

\begin{enumerate}

\item {\bf Claims}
    \item[] Question: Do the main claims made in the abstract and introduction accurately reflect the paper's contributions and scope?
    \item[] Answer: \answerYes{} % Replace by \answerYes{}, \answerNo{}, or \answerNA{}.
    \item[] Justification: The abstract and introduction in this paper accurately reflect the paper's contributions and scope.

\item {\bf Limitations}
    \item[] Question: Does the paper discuss the limitations of the work performed by the authors?
    \item[] Answer: \answerYes{}  % Replace by \answerYes{}, \answerNo{}, or \answerNA{}.
    \item[] Justification: We discuss the limitations of the work in Appendix \ref{limitations}.

\item {\bf Theory assumptions and proofs}
    \item[] Question: For each theoretical result, does the paper provide the full set of assumptions and a complete (and correct) proof?
    \item[] Answer:\answerYes{}  % Replace by \answerYes{}, \answerNo{}, or \answerNA{}.
    \item[] Justification: In this paper, all proofs of theorems are provided and all assumptions are clearly stated or referenced in the statement of any theorems.

    \item {\bf Experimental result reproducibility}
    \item[] Question: Does the paper fully disclose all the information needed to reproduce the main experimental results of the paper to the extent that it affects the main claims and/or conclusions of the paper (regardless of whether the code and data are provided or not)?
    \item[] Answer: \answerYes{} % Replace by \answerYes{}, \answerNo{}, or \answerNA{}.
    \item[] Justification: :  All the results in this paper can be reproduced.

\item {\bf Open access to data and code}
    \item[] Question: Does the paper provide open access to the data and code, with sufficient instructions to faithfully reproduce the main experimental results, as described in supplemental material?
    \item[] Answer: \answerYes{}{} % Replace by \answerYes{}, \answerNo{}, or \answerNA{}.
    \item[] Justification: The code
is available in \url{https://anonymous.4open.science/r/Materials}.

\item {\bf Experimental setting/details}
    \item[] Question: Does the paper specify all the training and test details (e.g., data splits, hyperparameters, how they were chosen, type of optimizer, etc.) necessary to understand the results?
    \item[] Answer: \answerYes{} % Replace by \answerYes{}, \answerNo{}, or \answerNA{}.
    \item[] Justification: The paper specify all the training and test details necessary in Section \ref{sec:experi} to understand the results.

\item {\bf Experiment statistical significance}
    \item[] Question: Does the paper report error bars suitably and correctly defined or other appropriate information about the statistical significance of the experiments?
    \item[] Answer:  \answerNo{} % Replace by \answerYes{}, \answerNo{}, or \answerNA{}.
    % \item[] Justification: Due to the resource limitation, we do not report error bars. We think the error bars are not related to the core result of our experiments.
    \item[] Justification: Due to the resource limitation, we do not report error bars. We think the error bars are not related to the core result of our experiments.

\item {\bf Experiments compute resources}
    \item[] Question: For each experiment, does the paper provide sufficient information on the computer resources (type of compute workers, memory, time of execution) needed to reproduce the experiments?
    \item[] Answer: \answerYes{} % Replace by \answerYes{}, \answerNo{}, or \answerNA{}.
    \item[] Justification: For each experiment, the paper provide sufficient information on the computer resources needed to reproduce the experiments. We provide them in Section \ref{Experiments_Settings}.
    
\item {\bf Code of ethics}
    \item[] Question: Does the research conducted in the paper conform, in every respect, with the NeurIPS Code of Ethics \url{https://neurips.cc/public/EthicsGuidelines}?
    \item[] Answer: \answerYes{} % Replace by \answerYes{}, \answerNo{}, or \answerNA{}.
    \item[] Justification: The research conducted in the paper conform, in every respect, with the NeurIPS Code of Ethics.

\item {\bf Broader impacts}
    \item[] Question: Does the paper discuss both potential positive societal impacts and negative societal impacts of the work performed?
    \item[] Answer: \answerYes{} % Replace by \answerYes{}, \answerNo{}, or \answerNA{}.
    \item[] Justification: We discuss the impacts of the work in Appendix \ref{impacts}.
    
\item {\bf Safeguards}
    \item[] Question: Does the paper describe safeguards that have been put in place for responsible release of data or models that have a high risk for misuse (e.g., pretrained language models, image generators, or scraped datasets)?
    \item[] Answer: \answerNA{} % Replace by \answerYes{}, \answerNo{}, or \answerNA{}.
    \item[] Justification: This paper does not deal with data or models with a high risk of misuse.

\item {\bf Licenses for existing assets}
    \item[] Question: Are the creators or original owners of assets (e.g., code, data, models), used in the paper, properly credited and are the license and terms of use explicitly mentioned and properly respected?
    \item[] Answer: \answerYes{} % Replace by \answerYes{}, \answerNo{}, or \answerNA{}.
    \item[] Justification: The creators or original owners of assets used in the paper are properly credited and the license and terms of use are explicitly mentioned and properly respected.

\item {\bf New assets}
    \item[] Question: Are new assets introduced in the paper well documented and is the documentation provided alongside the assets?
    \item[] Answer: \answerNA{} % Replace by \answerYes{}, \answerNo{}, or \answerNA{}.
    \item[] Justification: The paper does not release new assets.

\item {\bf Crowdsourcing and research with human subjects}
    \item[] Question: For crowdsourcing experiments and research with human subjects, does the paper include the full text of instructions given to participants and screenshots, if applicable, as well as details about compensation (if any)? 
    \item[] Answer: \answerNA{} % Replace by \answerYes{}, \answerNo{}, or \answerNA{}.
    \item[] Justification: The paper does not involve crowdsourcing nor research with human subjects

\item {\bf Institutional review board (IRB) approvals or equivalent for research with human subjects}
    \item[] Question: Does the paper describe potential risks incurred by study participants, whether such risks were disclosed to the subjects, and whether Institutional Review Board (IRB) approvals (or an equivalent approval/review based on the requirements of your country or institution) were obtained?
    \item[] Answer: \answerNA{} % Replace by \answerYes{}, \answerNo{}, or \answerNA{}.
    \item[] Justification: The paper does not involve crowdsourcing nor research with human subjects.

\item {\bf Declaration of LLM usage}
    \item[] Question: Does the paper describe the usage of LLMs if it is an important, original, or non-standard component of the core methods in this research? Note that if the LLM is used only for writing, editing, or formatting purposes and does not impact the core methodology, scientific rigorousness, or originality of the research, declaration is not required.
    %this research? 
    \item[] Answer: \answerNA{} % Replace by \answerYes{}, \answerNo{}, or \answerNA{}.
    \item[] Justification: The core method development in this research does not involve LLMs as any important, original, or non-standard components.

\end{enumerate}

\newpage
\appendix

\section{Notations}
\label{appendix:Notations}
\begin{table*}[h]
  \setlength{\tabcolsep}{20pt}
  \renewcommand{\arraystretch}{1.2}
  \centering
  \begin{adjustbox}{max width=\textwidth}
  \begin{tabular}{lcr}
    \textbf{Symbol} & \textbf{Description} 
    \\
    $\cT$ & \text{target task} 
    \\
    $\left\{\cS_1,\dots,\cS_K\right\}$ & \text{source tasks} 
    \\
    $N_0$ & \text{ quantity of target samples}
    \\
    $N_1 \cdots N_K$ & \text{maximum sample quantity of each source}
    \\
    $n_1 \cdots n_K$ & \text{transfer quantity of each source} 
     \\
    $\theta$ & model parameter\\
    $\underline{\theta}$ &  vectorized model parameter \\
    $\underline{\theta}_0$ &  vectorized model parameter of target task\\
    $\underline{\theta}_i, i\in[1,K]$ &  vectorized model parameter of i-th source task, \\
    $J(\theta)$  & Fisher information (scalar) of $\theta$ \\
    $J(\underline{\theta})^{d \times d}$  & Fisher information (matrix) of d-dimensional $\underline{\theta}$ \\
    % $P(X;\theta)$ & probability of X parameterized by $\theta$ \\
    $P_{X;\theta}$  & distribution of X parameterized by $\theta$  \\
    $|| \underline{x} ||^2$ & l-2 norm of vector x \\
    % $t(\theta_0, \theta_1, d)$ & task similarity \\
    $\alpha_i$ & transfer proportion in multi-source case from i-th source task\\
    $\underline{\alpha}$ & transfer proportion vector in multi-source case whose i-th entry is $\alpha_i$\\
    $s$ & total transfer quantity in multi-source case
    \\
    $\hat{\theta}$ & estimator of $\theta$ 
    \\
    $E_{\hat{\theta}}$ & expectation of $\hat{\theta}$ 
  \end{tabular}
  \end{adjustbox}
  \caption{Notations}
  \label{Notations}
\end{table*}

\newpage
\section{Extended Related Work}
\label{appendix_related_work}
\subsection{Transfer Learning Theory}
Existing theoretical works can be categorized into two groups. The first group focuses on proposing measures to quantify the similarity between the target and source tasks. Within this group, some measures have been introduced, including $l_2$-distance  \cite{long2014transfer}, optimal transport cost  \cite{courty2016optimal}, 
% K-L divergence~\cite{ganin2016domain,tzeng2017adversarial}, 
LEEP~\cite{nguyen2020leep}, Wasserstein distance~\cite{shui2021aggregating_WADN}, and maximal correlations~\cite{lee2019learning_MCW}. 
This work belongs to the second group focusing on developing new generalization error measures. Within this group, the measures 
having been introduced include $f$-divergence~\cite{harremoes2011pairs},  mutual information 
~\cite{bu2020tightening}, $\X^2$-divergence~\cite{tong2021mathematical}, $\mathcal{H}$-score~\cite{bao2019information,wu2024h_Hensemble}. 
However, the potential of K-L divergence as a generalization error measure has not been sufficiently explored. 
% In addition, some theoretical studies yield target models with constrained spaces, such as a convex combination of source models~\cite{tong2021mathematical}. 

\subsection{Multi-source Transfer Learning}
Classified by the object of transfer, existing multi-source transfer learning methods mainly focus on two types: model transfer vs sample transfer~\cite{zhuang2020comprehensive}. Model transfer assumes there is one or more pre-trained models on the source tasks and transfers their parameters to the target task via fine-tuning~\cite{wan2022uav}.  This work focuses on the latter, which is based on joint training of the source task samples with those of the target task~\cite{zhang2024revisiting_MADA,shui2021aggregating_WADN,li2021dynamic}. %Sample transfer can achieve better performance than model-based method as joint training takes full advantage of the information in the source data related to the target task. 
Classified by the strategy of transfer, existing methods mainly focus on two types: alignment strategy and matching-based strategy \cite{zhao2024more}. Alignment strategy aims to reduce the domain shift among source and target domains~\cite{li2021multi,zhao2021madan,li2021dynamic}. This work is more similar to the latter, focusing on determining which source domains or samples should be selected or assigned higher weights for transfer~\cite{guo2020multi,shui2021aggregating_WADN,tong2021mathematical,wu2024h_Hensemble}. 
However, most existing works either utilize all samples from all sources or perform task-level selection, whereas this work explores a framework that optimizes the transfer quantity of each source task.
% Classified by the object of transfer, existing multi-source transfer learning methods mainly focus on two types: model transfer vs sample transfer~\cite{zhuang2020comprehensive}. Model transfer assumes there is one or more pre-trained models on the source tasks and transfers their parameters to the target task via fine-tuning~\cite{wan2022uav}.  This work focuses on the latter, which is based on joint training of the source task samples with those of the target task~\cite{zhang2024revisiting_MADA,shui2021aggregating_WADN,li2021dynamic}. %Sample transfer can achieve better performance than model-based method as joint training takes full advantage of the information in the source data related to the target task. 
% Classified by the strategy of transfer, existing methods mainly focus on two types: alignment strategy and matching-based strategy \cite{zhao2024more}. Alignment strategy aims to reduce the domain shift among source and target domains~\cite{li2021multi,zhao2021madan,li2021dynamic}. This work is more similar to the latter, focusing on determining which source domains or samples should be selected or assigned higher weights for transfer~\cite{guo2020multi,shui2021aggregating_WADN,tong2021mathematical,wu2024h_Hensemble}. 
% However, most existing works either utilize all samples from all sources or perform task-level selection, whereas this work explores a framework that optimizes the transfer quantity of each source task.
Moreover, many works are restricted to specific target tasks, such as classification, which limits their \emph{task generality}. 
In addition, many studies are mainly applicable to few-shot scenarios, and may suffer from negative transfer in non-few-shot settings, which limits their~\emph{shot generality}, as illustrated in Table~\ref{tab:comparison}.

\newpage

\section{Proofs}

% You can have as much text here as you want. The main body must be at most $8$ pages long.
% For the final version, one more page can be added.
% If you want, you can use an appendix like this one.  

% The $\mathtt{\backslash onecolumn}$ command above can be kept in place if you prefer a one-column appendix, or can be removed if you prefer a two-column appendix.  Apart from this possible change, the style (font size, spacing, margins, page numbering, etc.) should be kept the same as the main body.
%%%%%%%%%%%%%%%%%%%%%%%%%%%%%%%%%%%%%%%%%%%%%%%%%%%%%%%%%%%%%%%%%%%%%%%%%%%%%%%
%%%%%%%%%%%%%%%%%%%%%%%%%%%%%%%%%%%%%%%%%%%%%%%%%%%%%%%%%%%%%%%%%%%%%%%%%%%%%%%
\subsection{Proof of Lemma~\ref{thm:target_only}} \label{appendix:target_only}

\begin{lemma}\label{thm:kl2mse}

In the asymptotic case, the proposed measure \eqref{eq:KL1} and the mean squared error have the relation as follows.
\begin{align}\label{eq:kl2mse}
&\mathbb{E} \left[ D\left(P_{X;\theta_0}\middle\| P_{X;\hat{\theta}}\right) \right]\notag\\&=\frac{1}{2} {J}(\theta_0)  \text{MSE}(\hat{\theta})+o(\frac{1}{N_0}),
\end{align}
\end{lemma}
\begin{proof}
In this section, for the sake of clarity, we will write $ \hat{\theta}$ in its parameterized form $\hat{\theta}(X^{N_0})$ when necessary, and these two forms are mathematically equivalent.
%First, we assume that $\vert\hat{\theta}(X^{N_0}) - \theta_0\vert^2=\epsilon^2$. 
By taking Taylor expansion of $D\left(P_{X;\theta_0}\middle\| P_{X;\hat{\theta}(X^{N_0})} \right)$ at $\theta_0$, we can get

\begin{align}
\label{appendix:taylor1}
&D\left( P_{X;\theta_0}\middle\| P_{X;\hat{\theta}(X^{N_0})}\right)
\notag\\&=\sum_{x\in X}P_{X;\theta_0}(x)\log\frac{P_{X;\theta_0}(x)}{P_{X;\hat{\theta}(X^{N_0})}(x)}
\notag\\&=-\sum_{x\in X}P_{X;\theta_0}(x)\log\frac{P_{X;\hat{\theta}(X^{N_0})}(x)}{P_{X;\theta_0}(x)}
\notag\\&=-\sum_{x\in X}P_{X;\theta_0}(x)\log\left(1+\frac{P_{X;\hat{\theta}(X^{N_0})}(x)-P_{X;\theta_0}(x)}{{P_{X;\theta_0}(x)}}\right)
\notag\\&=-\sum_{x\in X}P_{X;\theta_0}(x)\left(\left(\frac{P_{X;\hat{\theta}(X^{N_0})}(x)-P_{X;\theta_0}(x)}{{P_{X;\theta_0}(x)}}\right)-\frac{1}{2}\left(\frac{P_{X;\hat{\theta}(X^{N_0})}(x)-P_{X;\theta_0}(x)}{{P_{X;\theta_0}(x)}}\right)^2\right)+o(\vert\hat{\theta}(X^{N_0}) - \theta_0\vert^2)  
\notag\\&= \frac{1}{2} \sum_{x\in X} \frac{\left(P_{X;\hat{\theta}(X^{N_0})}(x) - P_{X;\theta_0}(x)\right)^2}{P_{X;\theta_0}(x)}+o(\vert\hat{\theta}(X^{N_0}) - \theta_0\vert^2)  
\end{align}

% \begin{align}
% \label{appendix:taylor1}
% D\left(P_{X;\theta_0}\middle\| P_{X;\hat{\theta}(X^{N_0})} \right)= \frac{1}{2} \sum_{x\in X} \frac{\left(P_{X;\hat{\theta}(X^{N_0})}(x) - P_{X;\theta_0}(x)\right)^2}{P_{X;\theta_0}(x)}+o(\vert\hat{\theta}(X^{N_0}) - \theta_0\vert^2)   
% \end{align}
                                                                                           
We denote $\delta$ as a small constant,  and we can rewrite \eqref{eq:KL1} as 
\begin{align}
\label{appendix:expect}
&\mathbb{E} \left[ D\left(P_{X;\theta_0}\middle\| P_{X;\hat{\theta}(X^{N_0})} \right) \right] \notag\\ 
&= \sum_{X^{N_0}}P_{X^n;\theta_0}(X^{N_0})D\left(P_{X;\theta_0}\middle\| P_{X;\hat{\theta}(X^{N_0})} \right)
\notag\\ 
&= \sum_{X^{N_0}}P_{X^n;\theta_0}(X^{N_0}) \left (  \frac{1}{2} \sum_{x\in X} \frac{\left(P_{X;\hat{\theta}(X^{N_0})}(x) - P_{X;\theta_0}(x)\right)^2}{P_{X;\theta_0}(x)}+o(\vert\hat{\theta}(X^{N_0}) - \theta_0\vert^2)  \right )
\notag\\ 
&= \sum_{X^{N_0}}P_{X^n;\theta_0}(X^{N_0})\left (  \frac{1}{2} \sum_{x\in X} \frac{ \left(\frac{\partial P_{X;\theta_0}(x)}{\partial \theta_0} (\hat{\theta}(X^{N_0}) - \theta_0)\right)^2}{P_{X;\theta_0}(x)}+o(\vert\hat{\theta}(X^{N_0}) - \theta_0\vert^2)  \right)
\notag\\ 
% &= \sum_{\left\{X^{N_0}:D\left(P_{X;\theta_0}\middle\| P_{X;\hat{\theta}(X^{N_0})} \right)<\delta\right\}}P_{X^n;\theta_0}(X^{N_0})D\left(P_{X;\theta_0}\middle\| P_{X;\hat{\theta}(X^{N_0})} \right)\notag\\ 
%  &+\sum_{\left\{X^{N_0}:D\left(P_{X;\theta_0}\middle\| P_{X;\hat{\theta}(X^{N_0})} \right)\ge\delta\right\}}P_{X^n;\theta_0}(X^{N_0})D\left(P_{X;\theta_0}\middle\| P_{X;\hat{\theta}(X^{N_0})} \right)
% \notag\\ 
    &= \sum_{\left\{X^{N_0}:\vert\hat{\theta}(X^{N_0}) - \theta_0\vert^2<\delta\right\}}P_{X^n;\theta_0}(X^{N_0}) \left ( \frac{1}{2} \sum_{x\in X} \frac{ \left(\frac{\partial P_{X;\theta_0}(x)}{\partial \theta_0} (\hat{\theta}(X^{N_0}) - \theta_0)\right)^2}{P_{X;\theta_0}(x)}+o(\vert\hat{\theta}(X^{N_0}) - \theta_0\vert^2)  \right)\notag\\ 
    &+\sum_{\left\{X^{N_0}:\vert\hat{\theta}(X^{N_0}) - \theta_0\vert^2\ge\delta\right\}}P_{X^n;\theta_0}(X^{N_0})
    \left (  \frac{1}{2} \sum_{x\in X} \frac{ \left(\frac{\partial P_{X;\theta_0}(x)}{\partial \theta_0} (\hat{\theta}(X^{N_0}) - \theta_0)\right)^2}{P_{X;\theta_0}(x)}+o(\vert\hat{\theta}(X^{N_0}) - \theta_0\vert^2) \right)\notag\\ 
\end{align}
To facilitate the subsequent proof, we introduce the concept of "Dot Equal".
\begin{definition}(Dot Equal (\(\dot{=}\)))
    Specifically, given two functions \( f(n) \) and \( g(n) \), the notation \( f(n) \dot{=} g(n) \) is defined as
\begin{align}
    f(n) \dot{=} g(n) \quad \Leftrightarrow \quad \lim_{n \to \infty} \frac{1}{n}\log\frac{f(n)}{g(n)} = 0,
\end{align}
which shows that $f(n)$ and $g(n)$ have the same exponential decaying rate.
\end{definition}

We denote $\hat{P}_{X^{N_0}}$ as the empirical distribution of $X^{N_0}$. Applying Sanov's Theorem to \eqref{appendix:expect}, we can know that
\begin{align}
\label{appendix:dotequal_1}
% &\mathbb{E} \left[ D\left(P_{X;\theta_0}\middle\| P_{X;\hat{\theta}(X^{N_0})} \right) \right] \notag\\ 
% &=\sum_{\hat{\theta},\|\hat{\theta}(X^{N_0}) - \theta_0\|< \delta}e^{-N_{0}D\left(P_{X;\theta_0}\middle\| P_{X;\hat{\theta}(X^{N_0})} \right)}
%     \mathbb{E} \left[ \frac{1}{2}\sum_x \frac{(P_{X;\hat{\theta}} - P_{X;\theta_0})^2}{P_{X;\theta_0}}+o(\vert\hat{\theta}(X^{N_0}) - \theta_0\vert^2)  \right]\notag\\ 
%     &+\sum_{\hat{\theta},\|\hat{\theta}(X^{N_0}) - \theta_0\|\ge \delta}e^{-N_{0}D\left(P_{X;\theta_0}\middle\| P_{X;\hat{\theta}(X^{N_0})} \right)}
%     \mathbb{E} \left[ \frac{1}{2}\sum_x \frac{(P_{X;\hat{\theta}} - P_{X;\theta_0})^2}{P_{X;\theta_0}}+o(\vert\hat{\theta}(X^{N_0}) - \theta_0\vert^2)  \right]
P_{X^n;\theta_0}(X^{N_0})\doteq e^{-N_{0}D\left(\hat{P}_{X^{N_0}} \middle\| P_{X;\theta_0}\right)} 
%\doteq e^{-N_{0}\frac{1}{2} \sum_{x\in X} \frac{ \left(\frac{\partial P_{X;\theta_0}(x)}{\partial \theta_0} (\hat{\theta}(X^{N_0}) - \theta_0)\right)^2}{P_{X;\theta_0}(x)}} 
\end{align}
Then, we aim to establish a connection between 
%the $D\left(\hat{P}_{X^{N_0}} \middle\| P_{X;\theta_0}\right)$ in 
\eqref{appendix:dotequal_1} and $\vert\hat{\theta}(X^{N_0}) - \theta_0\vert^2$. From \eqref{appendix:taylor1}, we can know that the $D\left(\hat{P}_{X^{N_0}} \middle\| P_{X;\theta_0}\right)$ in \eqref{appendix:dotequal_1} can be transformed to
\begin{align}
\label{appendix:taylor2}
&D\left( \hat{P}_{X^{N_0}}\middle\| P_{X;\theta_0}\right)=\frac{1}{2} \sum_{x\in X} \frac{\left(\hat{P}_{X^{N_0}}(x) - P_{X;\theta_0}(x)\right)^2}{\hat{P}_{X^{N_0}}(x)}+o(\vert\hat{\theta}(X^{N_0}) - \theta_0\vert^2)  
\notag\\&= \frac{1}{2} \sum_{x\in X} \frac{\left(\hat{P}_{X^{N_0}}(x) - P_{X;\theta_0}(x)\right)^2}{P_{X;\theta_0}(x)}+o(\vert\hat{\theta}(X^{N_0}) - \theta_0\vert^2)   
\end{align}

From the characteristics of MLE, we can know that
\begin{align}
\label{appendix:mle_a1}
    &\mathbb{E}_{\hat{P}_{X^{N_0}}}\left[\frac{\partial\log P_{X;{\hat{\theta}(X^{N_0})}}(x)}{\partial \hat{\theta}}\right]
    \notag\\&=0   \notag\\&=\mathbb{E}_{\hat{P}_{X^{N_0}}}\left[\frac{\partial\log P_{X;{\theta_{0}}}(x)}{\partial \theta_{0}}\right]+\mathbb{E}_{\hat{P}_{X^{N_0}}}\left[\frac{\partial^{2}\log P_{X;{\theta_{0}}}(x)}{\partial \theta_0^{2}}\right]{\left(\hat{\theta}(X^{N_0})-\theta_0\right)}+O(\vert\hat{\theta}(X^{N_0}) - \theta_0\vert^2)  ,
\end{align}
which can be transform to
\begin{align}
\label{appendix:mle_a2}
&\left(\hat{\theta}(X^{N_0})-\theta_0\right)+O(\vert\hat{\theta}(X^{N_0}) - \theta_0\vert^2)
\notag\\&=-\frac{\mathbb{E}_{\hat{P}_{X^{N_0}}}\left[\frac{\partial\log P_{X;{\theta_{0}}}(x)}{\partial \theta_0}\right]}{\mathbb{E}_{\hat{P}_{X^{N_0}}}\left[\frac{\partial^{2}\log P_{X;{\theta_{0}}}(x)}{\partial \theta_0^{2}}\right]}
\notag\\&=-\frac{\mathbb{E}_{\hat{P}_{X^{N_0}}}\left[\frac{\frac{\partial P_{X;{\theta_{0}}}(x)}{\partial \theta_0}}{P_{X;{\theta_{0}}}(x)}\right]}{\mathbb{E}_{\hat{P}_{X^{N_0}}}\left[\frac{\partial^{2}\log P_{X;{\theta_{0}}}(x)}{\partial \theta_0^{2}}\right]}
\notag\\&=\frac{\sum\limits_{x\in\mathcal{X}}{\left(\hat{P}_{X^{N_0}}(x)- P_{X;\theta_0}(x)\right)}\frac{\frac{\partial P_{X;{\theta_{0}}}(x)}{\partial \theta_0}}{P_{X;{\theta_{0}}}(x)}}{J(\theta_{0})}
\end{align}

Using the Cauchy-Schwarz inequality, we can obtain
\begin{align}
\label{Cauchy}
\sum_{x\in X} \frac{\left(\hat{P}_{X^{N_0}}(x) - P_{X;\theta_0}(x)\right)^2}{P_{X;\theta_0}(x)}\cdot \sum_x \frac{(\frac{\partial P_{X;\theta_0}(x)}{\partial \theta_0})^2}{P_{X;\theta_0}(x)}\ge\left(\sum\limits_{x\in\mathcal{X}}{\left(\hat{P}_{X^{N_0}}(x)- P_{X;\theta_0}(x)\right)}\frac{\frac{\partial P_{X;{\theta_{0}}}(x)}{\partial \theta_0}}{P_{X;{\theta_{0}}}(x)}\right)^2
\end{align}
where
\begin{align}
\label{partial_to_fisher}
       \sum_x \frac{(\frac{\partial P_{X;\theta_0}(x)}{\partial \theta_0})^2}{P_{X;\theta_0}(x)}
       =\sum_x P_{X;\theta_0}(x) \left( \frac{1}{P_{X;\theta_0}(x)} \frac{\partial P_{X;\theta_0}(x)}{\partial \theta_0} \right)^2 
       =\sum_x P_{X;\theta_0}(x) \left( \frac{\partial \log P_{X;\theta_0}(x)}{\partial \theta_0} \right)^2 
       =J(\theta_0) 
\end{align}
Combining with \eqref{appendix:taylor2}, \eqref{appendix:mle_a2}, and \eqref{partial_to_fisher}, the inequality \eqref{Cauchy} can be transformed to 
\begin{align}
\label{Cauchy2}
&D\left( \hat{P}_{X^{N_0}}\middle\|P_{X;\theta_0} \right)
\notag\\&= \frac{1}{2} \sum_{x\in X} \frac{\left(\hat{P}_{X^{N_0}}(x) - P_{X;\theta_0}(x)\right)^2}{P_{X;\theta_0}(x)}+o(\vert\hat{\theta}(X^{N_0}) - \theta_0\vert^2)  \notag\\&\ge\frac{1}{2}J(\theta_0)\left(\hat{\theta}(X^{N_0})-\theta_0+O(\vert\hat{\theta}(X^{N_0}) - \theta_0\vert^2)\right)^2+o(\vert\hat{\theta}(X^{N_0}) - \theta_0\vert^2)  
\notag\\&=\frac{1}{2}J(\theta_0)\vert\hat{\theta}(X^{N_0}) - \theta_0\vert^2+o(\vert\hat{\theta}(X^{N_0}) - \theta_0\vert^2)  
\end{align}

Combining \eqref{appendix:dotequal_1} and \eqref{Cauchy2}, we can know that
\begin{align}
\label{appendix:dotequal_2}
    P_{X^n;\theta_0}(X^{N_0})\doteq e^{-N_{0}D\left(\hat{P}_{X^{N_0}} \middle\| P_{X;\theta_0}\right)} \leq e^{\frac{-N_{0}J(\theta_0)\vert\hat{\theta}(X^{N_0}) - \theta_0\vert^2}{2}} 
\end{align}

For the first term in \eqref{appendix:expect}, the $\|\hat{\theta}(X^{N_0}) - \theta_0\|$ is small enough for us to omit the term $o(\vert\hat{\theta}(X^{N_0}) - \theta_0\vert^2) $. As for the second term in in \eqref{appendix:expect}, even though the magnitude of $\|\hat{\theta}(X^{N_0}) - \theta_0\|$ is no longer negligible, the probability of such sequences is  $O( e^{\frac{-N_{0}J(\theta_0)\vert\hat{\theta}(X^{N_0}) - \theta_0\vert^2}{2}})$ by \eqref{appendix:dotequal_2}, which is exponentially decaying with $N_0$ such that the second term is $o{(\frac{1}{N_0})}$.
By transfering \eqref{appendix:expect}, we can get
\begin{align}
\label{appendix:expect3}
&\mathbb{E} \left[ D\left(P_{X;\theta_0}\middle\| P_{X;\hat{\theta}(X^{N_0})} \right) \right] \notag\\ 
&=     \sum_{\left\{X^{N_0}:\vert\hat{\theta}(X^{N_0}) - \theta_0\vert^2<\delta\right\}}P_{X^n;\theta_0}(X^{N_0})  \left(  \frac{1}{2} \sum_{x\in X} \frac{ \left(\frac{\partial P_{X;\theta_0}(x)}{\partial \theta_0} (\hat{\theta}(X^{N_0}) - \theta_0)\right)^2}{P_{X;\theta_0}(x)} \right)+o(\vert\hat{\theta}(X^{N_0}) - \theta_0\vert^2)  \notag\\   
     &+\sum_{\left\{X^{N_0}:\vert\hat{\theta}(X^{N_0}) - \theta_0\vert^2\ge\delta\right\}}P_{X^n;\theta_0}(X^{N_0})
     \left( \frac{1}{2} \sum_{x\in X} \frac{ \left(\frac{\partial P_{X;\theta_0}(x)}{\partial \theta_0} (\hat{\theta}(X^{N_0}) - \theta_0)\right)^2}{P_{X;\theta_0}(x)} \right)+o{(\frac{1}{N_0})}\notag\\ 
&=\frac{1}{2} \mathbb{E} \left[ \sum_{x\in X} \frac{ \left(\frac{\partial P_{X;\theta_0}(x)}{\partial \theta_0} (\hat{\theta}(X^{N_0}) - \theta_0)\right)^2}{P_{X;\theta_0}(x)} \right]+o{(\frac{1}{N_0})}.
\end{align}
%and we can analyze that when $\|\hat{\theta}(X^{N_0}) - \theta_0\|\ge \delta$, the $e^{-N_{0}D\left(P_{X;\hat{\theta}}\right)}$ decreases exponentially as $N_0$ increases.

%%%%%%%%%%%%%%%%%%%%%%%%%%%%%%%%%%%%%%%%%%%%%%%%%%%%%%%%%%%%%%%%%%%%
% where we know the $\sum\limits_{\left\{X^{N_0}:\vert\hat{\theta}(X^{N_0}) - \theta_0\vert^2<\delta\right\}}o\left(\vert\hat{\theta}(X^{N_0}) - \theta_0\vert^2\right)$ is $o{(\frac{1}{N_0})}$ by Lemma \ref{thm:Cramér}. 
%这里我们就是想证明，泰勒展开的高阶项不会有影响，o和O的区别，在离得进的时候本来就不会有影响，就是o，在离得远的时候，我们给了一个1/N_0的界。
%%%%%%%%%%%%%%%%%%%%%%%%%%%%%%%%%%%%%%%%%%%%%%%%%%%%%%%%%%%%%%%%%%%%%%%%%%

We then transform \eqref{appendix:expect3} with \eqref{partial_to_fisher}
\begin{align}
    \label{KL1_cont}
    %&\mathbb{E} \left[ D\left(P_{X;\theta_0}\middle\| P_{X;\hat{\theta}}\right) \right] \notag\\ 
    %= &\frac{1}{2} \mathbb{E} \left[ \sum_x \frac{(P_{X;\hat{\theta}} - P_{X;\theta_0})^2}{P_{X;\theta_0}(x)} \right] \notag\\ 
    % &\frac{1}{2} \mathbb{E} \left[ \sum_x \frac{\left(P_{X;\hat{\theta}}(x) - P_{X;\theta_0}(x)\right)^2}{P_{X;\theta_0}(x)} \right] \notag\\ 
    % &= \frac{1}{2}  \mathbb{E} \left[ \sum_x \frac{ \left(\frac{\partial P_{X;\theta_0}(x)}{\partial \theta_0} (\hat{\theta} - \theta_0)\right)^2}{P_{X;\theta_0}(x)} \right] 
    &\frac{1}{2} \mathbb{E} \left[  \sum_{x\in X} \frac{ \left(\frac{\partial P_{X;\theta_0}(x)}{\partial \theta_0} (\hat{\theta}(X^{N_0}) - \theta_0)\right)^2}{P_{X;\theta_0}(x)} \right]\notag\\
    &=  \frac{1}{2} \mathbb{E} \left[\left(\hat{\theta} - \theta_0\right)^2 \right] \sum_x \frac{(\frac{\partial P_{X;\theta_0}(x)}{\partial \theta_0})^2}{P_{X;\theta_0}(x)}  \notag\\
    % &= \frac{1}{2} \mathbb{E} \left[\left(\hat{\theta} - \theta_0\right)^2 \right] \sum_x P_{X;\theta_0}(x) \left( \frac{1}{P_{X;\theta_0}(x)} \frac{\partial P_{X;\theta_0}(x)}{\partial \theta_0} \right)^2 \notag\\
    % &= \frac{1}{2} \mathbb{E} \left[\left(\hat{\theta} - \theta_0\right)^2 \right]  \sum_x P_{X;\theta_0}(x) \left( \frac{\partial \log P_{X;\theta_0}(x)}{\partial \theta_0} \right)^2 \notag\\
    &= \frac{1}{2} \mathbb{E} \left[\left(\hat{\theta} - \theta_0\right)^2 \right]  {J}(\theta_0) 
\end{align}
%Since the MLE achieves the cramer-rao bound,  we have
%Since the MLE estimator can achieve a mean squared error approach the Cramer-Rao bound in the asymptotic case, we will mainly analyze this situation. we have
Combining  \eqref{appendix:expect3}, \eqref{KL1_cont}, we can get 
\begin{align}
\mathbb{E} \left[ D\left(P_{X;\theta_0}\middle\| P_{X;\hat{\theta}}\right) \right]=\frac{1}{2} \mathbb{E} \left[\left(\hat{\theta} - \theta_0\right)^2 \right]  {J}(\theta_0) +o{(\frac{1}{N_0})}.
\end{align}
\end{proof}

% Combining \eqref{eq:Cramér} and \eqref{eq:kl2mse}, 
% we can establish the relationship between the proposed proposed measure \eqref{eq:KL1} and the mean squared error, and provide a bound for the proposed proposed measure \eqref{eq:KL1}. In the asymptotic case, proposed measure approaches this bound. Therefore, we will primarily analyze this bound in the following sections, and we denote this objective bound as
% \begin{align}\label{eq:objectivebound}
% B\left(P_{X;\hat{\theta}} \middle\| P_{X;\theta_0}\right).
% \end{align}
From \eqref{eq:kl2mse}, 
we can establish the relationship between the proposed proposed measure \eqref{eq:KL1} and the mean squared error. Then, by using the \eqref{eq:Cramér2}, the K-L~measure is \begin{align}\frac{1}{2N_0}+o{(\frac{1}{N_0})}. \end{align}

\subsection{Proof of Theorem~\ref{thm:one_source}} \label{appendix:one_source}
In this theorem, we are using samples from both target task and source task for our maximum likelihood estimation, so our optimization problem becomes
\begin{align}
    \hat{\theta} = \argmax_{\theta} L_n(\theta),
\end{align}
where, when using $N_{0}$ samples from target task, and $n_1$ samples from source task
\begin{align}
    L_n(\theta) \defeq \frac{1}{N_{0}+n_1} \sum_{x \in X^{N_{0}}} \log P_{X;\theta}(x) +\frac{1}{N_{0}+n_1} \sum_{x \in X^{n_{1}}} \log P_{X;\theta}(x).
\end{align}
And, we also define the expectation of our estimator, which is somewhere between $\theta_0$ and $\theta_1$ to be $ E_{\hat{\theta}}$
\begin{align}
\label{eq: theta_3}
    E_{\hat{\theta}} = \argmax_{\theta} L(\theta),
\end{align}
where  $L(\theta)$ denotes the expectation of $L_n(\theta)$
\begin{align}
    L(\theta) \defeq \frac{N_{0}}{N_{0}+n_1} \mathbb{E}_{P_{X;\theta_0}}\left[ \log P_{X;\theta}(x)\right] +\frac{n_1}{N_{0}+n_1} \mathbb{E}_{P_{X;\theta_1}}\left[ \log P_{X;\theta}(x)\right].
\end{align}

We could equivalently transform \eqref{eq: theta_3} into
\begin{align}
    \label{eq:expectedloss2}
    E_{\hat{\theta}} = \argmin_{\theta}\frac{N_{0}D\left(P_{X;\theta_0} \middle\| P_{X;\theta}\right)}{N_{0}+n_1} +\frac{n_{1}D\left(P_{X;\theta_1} \middle\| P_{X;\theta}\right)}{N_{0}+n_1}
\end{align}

\begin{lemma}
By taking argmin of \eqref{eq:expectedloss2} we can get
\begin{align}
    \label{eq:ptheta3}
    P_{X;E_{\hat{\theta}}} = \frac{N_{0}P_{X;\theta_0} + n_{1}P_{X;\theta_1}}{N_{0}+n_1},
\end{align}

By doing a Taylor Expansion of \eqref{eq:ptheta3} around  $\theta'$, which is in the neighbourhood of $\theta_1, \theta_2$ and $ E_{\hat{\theta}}$, we can get
\begin{align}
\label{eq:theta3totheta12}
    E_{\hat{\theta}} = \frac{N_{0}\theta_0 + n_{1}\theta_1}{N_{0}+n_1} + O\left(\frac{1}{N_{0}+n_1}\right),
\end{align}
\end{lemma}
\begin{proof}
From \eqref{eq:expectedloss2}, we know that
\begin{align}
    \label{eq:expectedloss2_2}
    &E_{\hat{\theta}} = \argmin_{\theta}\frac{N_{0}D\left(P_{X;\theta_0} \middle\| P_{X;\theta}\right)}{N_{0}+n_1} +\frac{n_{1}D\left(P_{X;\theta_1} \middle\| P_{X;\theta}\right)}{N_{0}+n_1}
    \notag\\&= \argmin_{\theta}\frac{N_{0}}{N_{0}+n_1}\sum_{x \in \mathcal{X}} P_{X;\theta_0} (x) \log \frac{P_{X;\theta_0} (x)}{P_{X;\theta}(x)}+\frac{n_1}{N_{0}+n_1}\sum_{x \in \mathcal{X}} P_{X;\theta_1} (x) \log \frac{P_{X;\theta_1} (x)}{P_{X;\theta}(x)}
\end{align}
To minimize the weighted K-L divergence. We treat $P_{X;\theta}(x), \forall x\in \mathcal{X}$ as the variable with the constraint \( \sum\limits_x P_{X;\theta}(x) = 1 \), then we form the Lagrangian:
\begin{align}
\text{Lagrangian}(P, \lambda) =
- \sum_x \left( \frac{N_0}{N_0 + n_1} P_{X;\theta_0}(x)
+ \frac{n_1}{N_0 + n_1} P_{X;\theta_1}(x) \right) \log P_{X;\theta}(x)
+ \lambda \left( \sum_x P_{X;\theta}(x) - 1 \right)
\end{align}
Taking the derivative with respect to 
$P_{X;\theta}(x)$ and setting it to zero gives:
\begin{align}
\frac{\partial \text{Lagrangian}(P, \lambda)}{\partial P_{X;\theta}(x)} = -\frac{\frac{N_0}{N_0 + n_1} P_{X;\theta_0}(x) + \frac{n_1}{N_0 + n_1} P_{X;\theta_1}(x)}{P_{X;\theta}(x)} + \lambda = 0.
\end{align}
So
\begin{align}
P_{X;\theta}(x) = \frac{\frac{N_0}{N_0 + n_1} P_{X;\theta_0}(x) + \frac{n_1}{N_0 + n_1} P_{X;\theta_1}(x)}{\lambda}.
\end{align}

Normalizing \( P_{X;\theta}(x) \) gives \( \lambda = 1 \), hence the optimal solution is:
\begin{align}
\label{eq:beginprove37}
P_{X;E_{\hat{\theta}}}(x) = \frac{N_0}{N_0 + n_1} P_{X;\theta_0}(x) + \frac{n_1}{N_0 + n_1} P_{X;\theta_1}(x), \forall x\in \mathcal{X},
\end{align}
which corresponds to \eqref{eq:ptheta3}.

Then we begin to prove \eqref{eq:theta3totheta12}.  By doing a Taylor Expansion of \eqref{eq:beginprove37} around $\theta_0$, we can get
\begin{align}
\label{eq:1111}
&P_{X;\theta_0}(x)+\frac{\partial P_{X;\theta_0}(x)}{\partial \theta}(E_{\hat{\theta}}-\theta_0)+O(\vert E_{\hat{\theta}}-\theta_0\vert^2)\notag\\&=\frac{N_0}{N_0 + n_1} P_{X;\theta_0}(x) +\frac{n_1}{N_0 + n_1} \left(P_{X;\theta_0}(x)+\frac{\partial P_{X;\theta_0}(x)}{\partial \theta}(\theta_1-\theta_0)+O(\vert\theta_0-\theta_1\vert^2)\right)
\end{align}
From \eqref{eq:1111} we can get
\begin{align}
\label{eq:3737}
    E_{\hat{\theta}} = \frac{N_{0}\theta_0 + n_{1}\theta_1}{N_{0}+n_1} + O(\vert\theta_0-\theta_1\vert^2),
\end{align}

% We have assumed that $\vert\theta_0-\theta_1\vert=O(\frac{1}{\sqrt{N_{0}}})=O(\frac{1}{\sqrt{N_{0}+n_1}})$ in lines 211-219 in our paper. 
So we can get 
\begin{align}
\label{eq:373737}
    E_{\hat{\theta}} = \frac{N_{0}\theta_0 + n_{1}\theta_1}{N_{0}+n_1} + O\left(\frac{1}{N_{0}+n_1}\right).
\end{align}

\end{proof}

% \begin{lemma}
% % \begin{align}
% % \mathbb{E}\left[ \left(\hat{\theta} -  E_{\hat{\theta}}\right)^2\right]=\frac{\frac{N_{0}}{N_{0}+n_1}J(\theta_0)+\frac{n_1}{N_{0}+n_1}J(\theta_1)}{(N_{0}+n_1)J^2( E_{\hat{\theta}})}
% % \end{align}
% \begin{align}
% \label{square_distribution}
% \sqrt{N_{0}+n_1}(\hat{\theta}- E_{\hat{\theta}})\xrightarrow{d}\mathcal{N}\left(0,\frac{1}{J(\theta_0)}\right)
% \end{align}
% \end{lemma}

\begin{lemma}
% Let \( \hat{\theta} \) denotes a MLE estimator using $N_0$ training samples $X^{N_0}$ generated from the target task distribution $P_{X;\theta_0}$ and $n_1$ training samples $X^{n_1}$ generated from the source task distribution $P_{X;\theta_1}$, i.e., 
% \begin{align}
%     \hat{\theta} = \argmax_{\theta} \frac{1}{N_{0}+n_1} \sum_{x \in X^{N_{0}}} \log P_{X;\theta}(x) +\frac{1}{N_{0}+n_1} \sum_{x \in X^{n_{1}}} \log P_{X;\theta}(x),
% \end{align}
% The notation \( E_{\hat{\theta}} \) denotes the expectation of \( \hat{\theta} \). The function \( J(\theta_0) \) represents the Fisher information evaluated at \( \theta_0 \). 
We assume that the following regularity conditions hold:

1. The log-likelihood function is twice continuously differentiable in the neighborhood of $\theta_0$.

2. The Fisher information $J(\theta_0)$ is positive and finite.

Then, the estimator 
$\hat{\theta}$ is asymptotically normal,i.e., 
\begin{align}
\label{square_distribution}
\sqrt{N_{0}+n_1}(\hat{\theta}- E_{\hat{\theta}})\xrightarrow{d}\mathcal{N}\left(0,\frac{1}{J(\theta_0)}\right).
\end{align}
\end{lemma}

\begin{proof}
Since $\hat{\theta}$ is a maximizer of $L_n{(\theta)}$, $L_n^{'}{(\hat{\theta})}=0=L_n^{'}{( E_{\hat{\theta}})}+L_n^{''}{( E_{\hat{\theta}})}(\hat{\theta}- E_{\hat{\theta}})+O\left(\frac{1}{N_{0}}\right)$. Therefore, 
\begin{align}
\label{eq11}
    \sqrt{N_{0}+n_1}(\hat{\theta}- E_{\hat{\theta}})=-\frac{\sqrt{N_{0}+n_1}L_n^{'}{( E_{\hat{\theta}})}}{L_n^{''}{( E_{\hat{\theta}})}}+O\left(\frac{1}{\sqrt{N_{0}}}\right).
\end{align}
Since $ E_{\hat{\theta}}$ maximizes $L{(\theta)}$, $L^{'}{( E_{\hat{\theta}})}=\frac{N_{0}}{N_{0}+n_1}\mathbb{E}_{\theta_0}\left[ \frac{\partial \log{P_{X;E_{\hat{\theta}}}(x)}}{\partial x} \right]+\frac{n_1}{N_{0}+n_1}\mathbb{E}_{\theta_1}\left[ \frac{\partial \log{P_{X;E_{\hat{\theta}}}(x_i)}}{\partial x} \right]=0$. 
Therefore,
\begin{align}
\label{lntheta3_00}
&\sqrt{N_{0}+n_1} L_n^{'}( E_{\hat{\theta}})
\nonumber\\&= \sqrt{\frac{N_{0}}{N_{0}+n_1} }\left( \sqrt{\frac{1}{N_{0}}}\sum_{x \in X^{N_{0}}}\frac{\partial \log{P_{X;E_{\hat{\theta}}}(x)}}{\partial \theta}\right)+\sqrt{\frac{n_1}{N_{0}+n_1} }\left( \sqrt{\frac{1}{n_1}}\sum_{x \in X^{n_1}}\frac{\partial \log{P_{X;E_{\hat{\theta}}}(x)}}{\partial \theta}\right)
\nonumber\\&= \sqrt{\frac{N_{0}}{N_{0}+n_1} }\left( \sqrt{\frac{1}{N_{0}}}\sum_{x \in X^{N_{0}}}\frac{\partial \log{P_{X;E_{\hat{\theta}}}(x)}}{\partial \theta}\right)+\sqrt{\frac{n_1}{N_{0}+n_1} }\left( \sqrt{\frac{1}{n_1}}\sum_{x \in X^{n_1}}\frac{\partial \log{P_{X;E_{\hat{\theta}}}(x)}}{\partial \theta}\right)
\nonumber\\&-\frac{N_{0}}{\sqrt{{N_{0}+n_1} }}\mathbb{E}_{\theta_0}\left[ \frac{\partial \log{P_{X;E_{\hat{\theta}}}(x)}}{\partial x} \right]-\frac{n_1}{\sqrt{{N_{0}+n_1} }}\mathbb{E}_{\theta_1}\left[ \frac{\partial \log{P_{X;E_{\hat{\theta}}}(x)}}{\partial x} \right]
\nonumber\\&=\sqrt{\frac{N_{0}}{N_{0}+n_1} }\left( \sqrt{\frac{1}{N_{0}}}\sum_{x \in X^{N_{0}}}\frac{\partial \log{P_{X;E_{\hat{\theta}}}(x)}}{\partial \theta} -\sqrt{N_{0}}\mathbb{E}_{\theta_0}\left[ \frac{\partial \log{P_{X;E_{\hat{\theta}}}(x)}}{\partial \theta} \right]\right)
\nonumber\\&+\sqrt{\frac{n_1}{N_{0}+n_1} }\left( \sqrt{\frac{1}{n_1}}\sum_{x \in X^{n_1}}\frac{\partial \log{P_{X;E_{\hat{\theta}}}(x)}}{\partial \theta} -\sqrt{n_1}\mathbb{E}_{\theta_1}\left[ \frac{\partial \log{P_{X;E_{\hat{\theta}}}(x)}}{\partial \theta} \right]\right)
\end{align}

Applying the Central Limit Theorem to \eqref{lntheta3_00}, we can get
\begin{align}\label{lntheta3}
    %&\var\left(\sqrt{N_{0}+n_1} L_n^{'}( E_{\hat{\theta}})\right)\nonumber\\
    \sqrt{N_{0}+n_1} L_n^{'}( E_{\hat{\theta}})\xrightarrow{a.s.}\mathcal{N} \Bigg(0, 
    &\frac{N_0}{N_{0}+n_1}\left(\mathbb{E}_{\theta_0}\left[\left ( \frac{\partial \log{P_{X;E_{\hat{\theta}}}(x)}}{\partial \theta}  \right )^2 \right]-\mathbb{E}_{\theta_0}\left[\left ( \frac{\partial \log{P_{X;E_{\hat{\theta}}}(x)}}{\partial \theta}  \right ) \right]^2\right)+\nonumber\\
    &\frac{n_1}{N_{0}+n_1}\left(\mathbb{E}_{\theta_1}\left[\left ( \frac{\partial \log{P_{X;E_{\hat{\theta}}}(x)}}{\partial \theta}  \right )^2 \right]-\mathbb{E}_{\theta_1}\left[\left ( \frac{\partial \log{P_{X;E_{\hat{\theta}}}(x)}}{\partial \theta}  \right ) \right]^2\right) \Bigg)
\end{align}

By taking Taylor expansion of $ E_{\hat{\theta}}$ at $\theta_1$, we can get
\begin{align}
\label{lntheta31}
&\mathbb{E}_{\theta_0}\left[\left ( \frac{\partial \log{P_{X;E_{\hat{\theta}}}(x)}}{\partial \theta}  \right )^2\right]\nonumber\\
&=\mathbb{E}_{\theta_0}\left[\left ( \frac{\partial \log{P_{X;\theta_0}(x)}}{\partial \theta}+\frac{\partial }{\partial \theta} \frac{\partial \log{P_{X;\theta_0}(x)}}{\partial  \theta}( E_{\hat{\theta}}-\theta_0)+O(\frac{1}{N_{0}+n_1}) \right )^2\right]\nonumber\\
&=J(\theta_0)+( E_{\hat{\theta}}-\theta_0)\mathbb{E}_{\theta_0}\left [ \frac{\partial \log{P_{X;E_{\hat{\theta}}}(x)}}{\partial \theta}\frac{\partial^2  \log{P_{X;E_{\hat{\theta}}}(x)}}{\partial\theta^2 } \right ]+O(\frac{1}{N_{0}+n_1})\nonumber\\
&=J(\theta_0)+O(\frac{1}{\sqrt{N_{0}+n_1}})
\end{align}
and
\begin{align}
\label{lntheta32}
&\mathbb{E}_{\theta_0}\left[\left ( \frac{\partial \log{P_{X;E_{\hat{\theta}}}(x)}}{\partial \theta}  \right )\right]^2\nonumber\\
&=\mathbb{E}_{\theta_0}\left[\left ( \frac{\partial\log{P_{X;\theta_0}(x)} }{\partial \theta}+\frac{\partial }{\partial \theta} \frac{\partial \log{P_{X;\theta_0}(x)}}{\partial  \theta}( E_{\hat{\theta}}-\theta_0)+O(\frac{1}{N_{0}+n_1}) \right )\right]^2\nonumber\\
&=\mathbb{E}_{\theta_0}\left[\left ( \frac{\partial }{\partial \theta} \frac{\partial \log{P_{X;\theta_0}(x)}}{\partial  \theta}( E_{\hat{\theta}}-\theta_0)+O(\frac{1}{N_{0}+n_1}) \right )\right]^2\nonumber\\
&=( E_{\hat{\theta}}-\theta_0)^2E_{\theta_0}\left[\left ( \frac{\partial }{\partial \theta} \frac{\partial \log{P_{X;\theta_0}(x)}}{\partial  \theta} \right )\right]^2+o(\frac{1}{N_{0}+n_1})\nonumber\\&=O(\frac{1}{N_{0}+n_1})
\end{align}
By combining \eqref{lntheta3}, \eqref{lntheta31}, and \eqref{lntheta32}, we can get
\begin{align}
\label{eq16}
\var\left(\sqrt{N_{0}+n_1} L_n^{'}( E_{\hat{\theta}})\right)=\frac{N_{0}}{N_{0}+n_1}J(\theta_0)+\frac{n_1}{N_{0}+n_1}J(\theta_1)+O(\frac{1}{\sqrt{N_{0}+n_1}})
\end{align}

Additionally, we know $L_n^{''}{( E_{\hat{\theta}})}\xrightarrow{p}-J( E_{\hat{\theta}})$. Combining with \eqref{eq11}\eqref{eq16}, we know that
\begin{align}
\label{square_distribution2_0}
\sqrt{N_{0}+n_1}(\hat{\theta}- E_{\hat{\theta}})\xrightarrow{d}\mathcal{N}\left(0,\frac{\frac{N_{0}}{N_{0}+n_1}J(\theta_0)+\frac{n_1}{N_{0}+n_1}J(\theta_1)}{J^2( E_{\hat{\theta}})}\right)
\end{align}
Under the assumption that $\theta_0, \theta_1, E_{\hat{\theta}}$ are sufficiently close to each other, we can easily deduce that the difference among ${J}(\theta_0),{J}(\theta_1)$ and ${J}(E_{\hat{\theta}})$ is $O(\frac{1}{\sqrt{N_{0}+n_1}})$. We can easily get
\begin{align}
\label{square_distribution2}
\sqrt{N_{0}+n_1}(\hat{\theta}- E_{\hat{\theta}})\xrightarrow{d}\mathcal{N}\left(0,\frac{1}{J(\theta_0)}\right)
\end{align}
\end{proof}

Therefore, the limit of $\mathbb{E}\left[ \left(\hat{\theta} - E_{\hat{\theta}}\right)^2\right]$ is
\begin{align}
\label{appendix:bound_single_source}
\frac{1}{(N_{0}+n_1)J( \theta_0)}
\end{align}
Combining \eqref{eq:KL1} and \eqref{KL1_cont}, we know that
\begin{align}
\label{appendix:kl_to_twokl}
    & \mathbb{E} \left[ D\left(P_{X;\theta_0}\middle\| P_{X;\hat{\theta}}\right) \right] \nonumber\\
    &=\frac{1}{2}{J}(\theta_0)\mathbb{E}\left[ \left(\hat{\theta} - \theta_0\right)^2\right]+o(\frac{1}{N_0})\nonumber\\
    &=\frac{1}{2}{J}(\theta_0)\bigg(\mathbb{E}\left[ \left(\hat{\theta} - E_{\hat{\theta}}\right)^2\right]+\mathbb{E}\left[ 2\left(\hat{\theta} - E_{\hat{\theta}}\right)\left(E_{\hat{\theta}} - \theta_0\right)\right]+\mathbb{E}\left[ \left(E_{\hat{\theta}} - \theta_0\right)^2\right]\bigg)+o(\frac{1}{N_0})\nonumber\\
    &=\frac{1}{2}{J}(\theta_0)\left(\mathbb{E}\left[ \left(\hat{\theta} - E_{\hat{\theta}}\right)^2\right]+\mathbb{E}\left[ \left(E_{\hat{\theta}} - \theta_0\right)^2\right]\right)+o(\frac{1}{N_0})
\end{align}
Combining \eqref{appendix:kl_to_twokl}, \eqref{eq:theta3totheta12} and \eqref{appendix:bound_single_source}, we know that the proposed measure  is
\begin{align}
% &=\frac{1}{2}\frac{1}{N_{0}+n_1}+\frac{1}{2}{J}(\theta_0)\left(E_{\hat{\theta}} - \theta_0\right)^2\nonumber\\
\frac{1}{2}\frac{1}{N_{0}+n_1}+\frac{1}{2}{J}(\theta_0)\frac{n_{1}^2}{(N_{0}+n_1)^2}\left(\theta_1 - \theta_0\right)^2+o(\frac{1}{N_0})
\end{align}

\subsection{Proof of Proposition~\ref{Proposition:target_only}} \label{appendix:prop_target_only}
Similar to \eqref{eq:kl2mse}, we can get
\begin{align}
& \mathbb{E} \left[ D\left( P_{X;{\underline{\theta}}_0}\middle\| P_{X;\hat{{\underline{\theta}}}}\right) \right] \nonumber\\
    &=\frac{1}{2}tr\left({J}({\underline{\theta}}_0)\mathbb{E}\left[ \left(\hat{{\underline{\theta}}} - {\underline{\theta}}_0\right)\left(\hat{{\underline{\theta}}} - {\underline{\theta}}_0\right)^T\right]\right)+o(\frac{1}{N_0})\nonumber\\
\end{align}
Combining with Lemma~ \ref{thm:Cramér}, we will know that the K-L~measure is
\begin{align}
\frac{1}{2}tr\left({J}({\underline{\theta}}_0)\left({J}({\underline{\theta}}_0)^{-1}\frac{1}{N_0}+o(\frac{1}{N_0})\right)\right)=\frac{d}{2N_0}+o(\frac{1}{N_0})
\end{align}
\subsection{Proof of Proposition~\ref{Proposition:one_source}} \label{appendix:Proposition:one_source}
Similar to \eqref{appendix:kl_to_twokl}, we can get
\begin{align}
\label{appendix:kl_to_twokl_prop2}
    & \mathbb{E} \left[ D\left( P_{X;{\underline{\theta}}_0}\middle\| P_{X;\hat{{\underline{\theta}}}}\right) \right] \nonumber\\
    &=\frac{1}{2}tr\left({J}({\underline{\theta}}_0)\mathbb{E}\left[ \left(\hat{{\underline{\theta}}} - {\underline{\theta}}_0\right)\left(\hat{{\underline{\theta}}} - {\underline{\theta}}_0\right)^T\right]\right)+o(\frac{1}{N_0})\nonumber\\
    &=\frac{1}{2}\left(tr\left({J}({\underline{\theta}}_0)\mathbb{E}\left[ \left(\hat{{\underline{\theta}}} - E_{\hat{{\underline{\theta}}}}\right)\left(\hat{{\underline{\theta}}} - E_{\hat{{\underline{\theta}}}}\right)^T\right]\right)+tr\left({J}({\underline{\theta}}_0)\mathbb{E}\left[ \left(E_{\hat{{\underline{\theta}}}} - {\underline{\theta}}_0\right)\left(E_{\hat{{\underline{\theta}}}} - {\underline{\theta}}_0\right)^T\right]\right)\right)+o(\frac{1}{N_0})
\end{align}
Similar to \eqref{square_distribution},we can get
\begin{align}
\label{square_distribution_one_source}
\sqrt{N_0+n_1}\left(\hat{{\underline{\theta}}} - E_{\hat{{\underline{\theta}}}}\right)\xrightarrow{d}\mathcal{N}\left(0,J({\underline{\theta}}_0)^{-1}\right)
\end{align}
So we can know that $\mathbb{E}\left[ \left(\hat{{\underline{\theta}}} - E_{\hat{{\underline{\theta}}}}\right)^2\right]$ has the limit
\begin{align}
\label{one_source_1_bound}
&\frac{1}{(N_0+n_1)}J({\underline{\theta}}_0)^{-1}
\end{align}
The same to \eqref{eq:theta3totheta12}, we can get
\begin{align}
\label{eq:theta3totheta12_prop}
    E_{\hat{{\underline{\theta}}}} = \frac{N_{0}{\underline{\theta}}_0 + n_{1}{\underline{\theta}}_1}{N_{0}+n_1} + O\left(\frac{1}{N_{0}+n_1}\right),
\end{align}
Combining \eqref{appendix:kl_to_twokl_prop2}, \eqref{one_source_1_bound} and \eqref{eq:theta3totheta12_prop}, we can know that the K-L~measure is
\begin{align} 
    %&=\frac{1}{2}\frac{d}{m+n}+\frac{1}{2}tr({J}(\vec{{\underline{\theta}}_1})\left(\vec{{\underline{\theta}}_2} - \vec{{\underline{\theta}}_1}\right)^2)\frac{m^2}{(m+n)^2}\nonumber\\
    \frac{d}{2}\left(\frac{1}{N_0+n_{1}}+\frac{n_{1}^2}{(N_0+n_{1})^2}t^{}\right) + o\left(\frac{1}{N_{0}+n_1}\right),
\end{align}
where we denote
\begin{align}
    t^{} \defeq \frac{({{\underline{\theta}}_1} - {{\underline{\theta}}_0})^{T}J({{\underline{\theta}}_0})({{\underline{\theta}}_1} - {{\underline{\theta}}_0})}{d}.
\end{align}
\subsection{Proof of Theorem~\ref{thm:multi_source}} \label{appendix:multi_source}
\begin{proof}
Similar to \eqref{square_distribution},we can get
\begin{align}
\label{square_distribution_multi_source}
% \sqrt{N_0+\sum\limits_{i=1}^{K}n_i}\left(\hat{\theta} - E_{\hat{\theta}}\right)\xrightarrow{d}\mathcal{N}\left(0,\frac{\frac{N_0}{N_0+\sum\limits_{i=1}^{K}n_i}J(\theta_0)+\sum\limits_{j=1}^{K}\frac{n_j}{N_0+\sum\limits_{i=1}^{K}n_i}J(\theta_j)}{J^2(E_{\hat{\theta}})}\right)
\sqrt{N_0+\sum\limits_{i=1}^{K}n_i}\left(\hat{{\underline{\theta}}} - E_{\hat{{\underline{\theta}}}}\right)\xrightarrow{d}\mathcal{N}\left(0,J({\underline{\theta}}_0)^{-1}\right)
\end{align}
So we can know that $\mathbb{E}\left[ \left(\hat{{\underline{\theta}}} - E_{\hat{{\underline{\theta}}}}\right)^2\right]$ has the limit
\begin{align}
\label{multi_1_bound}
&\frac{1}{(N_0+s)}J({\underline{\theta}}_0)^{-1}
\end{align}
Similar to \eqref{eq:theta3totheta12},we can get
\begin{align}
\label{multi_1_expectation}
E_{\hat{{\underline{\theta}}}}=\frac{N_0{\underline{\theta}}_0+\sum\limits_{i=1}^{k}n_i{\underline{\theta}}_i}{N_0+\sum\limits_{i=1}^{K}n_i}+O\left(\frac{1}{N_0+\sum\limits_{i=1}^{K}n_i}\right)=\frac{N_0{\underline{\theta}}_0+\sum\limits_{i=1}^{k}n_i{\underline{\theta}}_i}{N_0+s}+O\left(\frac{1}{N_{0}+s}\right)
\end{align}

% Combining \eqref{appendix:kl_to_twokl}, \eqref{multi_1_bound} and\eqref{multi_1_expectation} , we can know that the K-L~measure is
% \begin{align}
%  \frac{1}{2}\frac{1}{N_0+s}+\frac{1}{2}\frac{s^2}{(N_0+s)^2}J({\underline{\theta}}_0)\left(\sum\limits_{i=1}^{k}\alpha_i({\underline{\theta}}_i-{\underline{\theta}}_0)\right)^2+ o\left(\frac{1}{N_{0}+s}\right)
% \end{align} 

Combining \eqref{appendix:kl_to_twokl_prop2}, \eqref{multi_1_bound} and\eqref{multi_1_expectation} , for the d-demension parameter, we can know that the K-L~measure is 
\begin{align}
\label{appendix:eq:multisourcetarget}
&\frac{d}{2}\left(\frac{1}{N_0+s}+\frac{s^2}{(N_0+s)^2d}tr\left(\mathbf{J}(\underline{\theta}_0)\left(\sum\limits_{i=1}^{k}\alpha_i(\underline{\theta}-\underline{\theta}_0)\right)\left(\sum\limits_{i=1}^{k}\alpha_i(\underline{\theta}_i-\underline{\theta}_0)\right)^T\right)\right)\nonumber\\
&=\frac{d}{2}\left(\frac{1}{N_0+s}+\frac{s^2}{(N_0+s)^2}\frac{\underline{\alpha}^T\Theta^T{J}(\underline{\theta}_0)\Theta\underline{\alpha}}{d}\right)+ o\left(\frac{1}{N_{0}+s}\right)\nonumber\\
\end{align}
where $\underline{\alpha}=\left[\alpha_1,\dots,\alpha_K\right]^T$ is a K-dimensional vector, and 
\begin{align}
{{\Theta}}^{d\times K}=\left[{{\underline{\theta}}_1}-{{\underline{\theta}}_0},\dots,{{\underline{\theta}}_K}-{{\underline{\theta}}_0} \right].
\end{align}
\end{proof}    
\newpage
\section{Experiment Details}
\label{appendix:Experiment_Details}
\subsection{Details information of LoRA framework experiments.}
\begin{table}[!h]
\centering
\caption{Multi-Source Transfer with LoRA on Office-Home. We apply LoRA on ViT-B backbone for PEFT.}
\resizebox{0.48\textwidth}{!}{
\begin{tabular}{l c ccccc}
\toprule
\multirow{2.5}{*}{\textbf{Method}} & \multirow{2.5}{*}{\textbf{Backbone}} & \multicolumn{5}{c}{\textbf{Office-Home}}\\ 
\cmidrule(lr){3-7} 
&& $\to$Ar & $\to$Cl & $\to$Pr & $\to$Rw & Avg\\ \midrule
% \midrule
% \textit{Source-combine:} &&&&&&&&&&&&\\
% \textit{Supervised-10-shots Source-Ablation:} &&&&&\\
\multicolumn{7}{l}{\textit{\textbf{Supervised-10-shots Source-Ablation:}}} \\
Target-Only & ViT-B & 59.8 & 42.2 & 69.5 & 72.0 & 60.9 \\ 
Single-Source-avg & ViT-B & 72.2  & 59.9 & 82.6 & 81.0 & 73.9 \\
Single-Source-best & ViT-B & 74.4  & 61.8 & 84.9 & 81.9 & 75.8 \\ 
\allsource{} & ViT-B & \underline{81.1} & \underline{66.0} & \underline{88.0} & \underline{89.2} & \underline{81.1}\\
\ourmethod{} (Ours) & ViT-B & \textbf{81.5} & \textbf{68.0} & \textbf{89.2} & \textbf{90.3}  & \textbf{82.3} \\
\bottomrule
\end{tabular}
}
\label{tab:LoRA}
\end{table}
\subsection{Experimental Design and Model Adaptation}
\label{appendix:Experimental_Design_and_Model Adaptation}

% \textbf{Experimental Design and Model Adaptation}. 
To ensure consistency in the experimental setup, we first evaluate the performance of different methods on the DomainNet and Office-Home datasets by adapting their settings to align with ours, such as the backbone, dataset, and early stopping criteria. Specifically, for the MADA method, we adjusted the preset of keeping 5\% of labeled target samples to 10-shots per class target samples while maintaining other conditions. And in turn, ours is adapted to the WADN settings, equipped with a 3-layer ConvNet and evaluated on Digits dataset.

%from prof. yang: Performance on Digits classification dataset. 
\begin{table}[htbp]
\centering
\caption{\textbf{Performance on Digits Dataset.} The arrows indicate transferring from the rest tasks. ``3Conv'' denotes the backbone with 3 convolution layers.}
\resizebox{0.48\textwidth}{!}{
\begin{tabular}{l c ccccc}
\toprule
\multirow{2.5}{*}{\textbf{Method}} & \multirow{2.5}{*}{\textbf{Backbone}} & \multicolumn{5}{c}{\textbf{Digits}}\\ 
\cmidrule(lr){3-7} 
&& $\to$mt & $\to$sv & $\to$sy & $\to$up & Avg\\ 
 % \multicolumn{1}{c|}{\makecell{\textbf{Method}}} &\multicolumn{1}{c|}{\makecell{\textbf{Backbone}}} &\multicolumn{5}{c}{\makecell{\textbf{Digits}\\\midrule$\to$mt\hspace{1.2em}$\to$sv\hspace{1.2em}$\to$sy\hspace{1.2em}$\to$up\hspace{1em}Avg}}\\
\midrule
% \midrule
% \textit{Source-combine:} &&&&&&&&&&&&\\
% \textit{Supervised-10-shots Source-Ablation:} &&&&&\\
% None-Source(Vitb)  & 59.8 & 42.2 & 69.5 & 72.0 & 60.9 \\ 
% Single-Source-avg(Vitb)& 72.2  & 59.9 & 82.6 & 81.0 & 73.9 \\
% Single-Source-best(Vitb) & 74.4  & 61.8 & 84.9 & 81.9 & 75.8 \\ 
\multicolumn{7}{l}{\textit{\textbf{Following settings of WADN:}}} \\
WADN\cite{shui2021aggregating_WADN} &3Conv& 88.3 & \textbf{70.6} & 81.5 & \textbf{90.5} & 82.7 \\
\allsource{} &3Conv & \underline{92.6} & \underline{67.1} & \underline{82.5} & 88.8 & \underline{82.8}\\
\ourmethod{} (Ours)&3Conv& \textbf{93.8} & \underline{67.1} & \textbf{83.3} & \underline{89.1}  & \textbf{83.3} \\
\bottomrule
\end{tabular}
}
\label{tab:adapted_experiments}
\end{table}
\newpage
\subsection{Single-Source transfer performance details.}

\begin{table*}[!h]
\centering
\caption{\textbf{Single-Source Transfer Performance on DomainNet.} The details accuracy information of the ‘‘Single-Source-$*$’’ lines of Table \ref{tab:major}.}
\resizebox{0.8\textwidth}{!}{
\begin{tabular}{lc c cccccccc}
\toprule
% \multirow{2}{*}{\makecell{\textbf{Method} \\ Second line}} &\multirow{2}{*}{\makecell{\textbf{Method} \\ Second line}} 
\multirow{2.5}{*}{\textbf{Target Domain}} & \multirow{2.5}{*}
{\textbf{Backbone}} && \multicolumn{7}{c}{\textbf{Source Domain}}\\ 
% \cline{3-9} \cline{10-14}
\cmidrule(lr){4-11}
 &&&Clipart & Infograph & Painting & Quickdraw & Real & Sketch & Avg & Best\\ 
\midrule
% \midrule
% \textit{Source-combine:} &&&&&&&&&&&&\\
% \textit{Unsupervised-all-shots:} &&&&&&&&&&&&&\\
Clipart & ViT-S && - & 46.5 & 55.4  & 30.3  & 60.2 & 59.7  & 50.4 & 60.2\\
Infograph & ViT-S && 25.6 & - & 25.3  & 7.3  & 28.0 & 24.4  & 22.1 & 28.0  \\ 
Painting & ViT-S && 49.6 & 47.3 & -  & 22.4  & 55.4 & 49.6  & 44.9 & 55.4\\
Quickdraw & ViT-S && 26.9 & 18.1 & 23.9  & -  & 25.9 & 28.4  & 24.7 & 28.4\\
Real & ViT-S && 64.6 & 62.2 & 66.0  & 38.5  & - & 62.7  & 58.8 & 66.0\\
Sketch & ViT-S && 49.7 & 40.9 & 48.1  & 26.1  & 47.8 & -  & 42.5 & 49.7\\
\bottomrule
\end{tabular}
}
\label{tab:appendix_single_source_domainnet}
\end{table*}

\begin{table*}[!h]
\centering
\caption{\textbf{Single-Source Transfer Performance on Office-Home.} The details accuracy information of the ‘‘Single-Source-$*$’’ lines of Table \ref{tab:major}.}
\resizebox{0.6\textwidth}{!}{
\begin{tabular}{lc c cccccc}
\toprule
\multirow{2.5}{*}{\textbf{Target Domain}} & \multirow{2.5}{*}
{\textbf{Backbone}}&& \multicolumn{5}{c}{\textbf{Source Domain}}\\ 
\cmidrule(lr){4-9}
 &&& Art & Clipart & Product & Real World & Avg & Best\\ 
\midrule
Art & ViT-S && -  & 61.7 & 61.1 & 72.9 & 65.2 & 72.9 \\
Clipart & ViT-S && 49.8  & - & 49.4 & 60.9 & 53.3 & 60.9 \\ 
Product & ViT-S && 68.8  & 73.7 & - & 80.7 & 74.4 & 80.7 \\
Real World & ViT-S && 70.9  & 72.3 & 74.8 & - & 72.7 & 74.8 \\
\bottomrule
\end{tabular}
}
\label{tab:appendix_single_source_officehome}
\end{table*}

\begin{table*}[!h]
\centering
\caption{\textbf{Single-Source Transfer Performance on Office-Home of LoRA.} The details accuracy information of the ‘‘Single-Source-$*$’’ lines of Table \ref{tab:LoRA}. ViT-B backbone is already frozen and equipped with small trainable LoRA layers.}
\resizebox{0.6\textwidth}{!}{
\begin{tabular}{lc c cccccc}
\toprule
\multirow{2.5}{*}{\textbf{Target Domain}} & \multirow{2.5}{*}
{\textbf{Backbone}}&& \multicolumn{5}{c}{\textbf{Source Domain}}\\ 
\cmidrule(lr){4-9}
 &&& Art & Clipart & Product & Real World & Avg & Best\\ 
\midrule
Art & ViT-B && -  & 68.4 & 74.4 & 73.8 & 72.2 & 74.4\\
Clipart & ViT-B && 58.3  & - & 59.6 & 61.8 & 59.9 & 61.8\\ 
Product & ViT-B && 82.0  & 80.8 & - & 84.9 & 82.6 & 84.9 \\
Real World & ViT-B && 81.0  & 80.3 & 81.9 & - & 81.0 &81.9\\
\bottomrule
\end{tabular}
}
\label{tab:appendix_single_source_officehome_lora}
\end{table*}

\subsection{Baselines experiments settings in Table \ref{tab:major}.}
Since the significant difference from unsupervised methods and supervised methods, the results of MSFDA\cite{shen2023balancingMSFDA}, DATE\cite{han2023discriminability_DATE} and M3SDA\cite{peng2019moment_M3SDA} are directly supported by their own article.

On all the experiments of the baselines, we take all the source samples of different domains from trainset into account. And types of the backbone are all pretrained on ImageNet21k\cite{deng2009imagenet}.

As for based on model-parameter-weighting few-shot methods: MCW\cite{lee2019learning_MCW} and H-ensemble\cite{wu2024h_Hensemble}, since they have not taken experiments on DomainNet and Office-Home datasets with ViT-Small backbone and 10-shot per class training samples, we take it by changing the backbone and samples condition while maintaining other configurations supported by their own work. And for fairness and efficiency, the well-trained source models from different domains of these methods are directly equipped with the same models trained by the first stage of our \ourmethod{} method respectively.

As for based on samples few-shot methods: We exam the WADN\cite{shui2021aggregating_WADN} method under the condition of limited label on target domain as this setting is similar to ours. And we also change the backbone and samples condition while maintaining other configurations of its report to realize the experiments on DomainNet and Office-Home datasets. It is apparent that the settings of MADA\cite{zhang2024revisiting_MADA} are quite different of ours, leveraging all the labeled source data and all the unlabeled target data with the few-shot labeled ones that are difficult to classify, which means all the target samples have been learned to some extent. So we not only decrease the labeled target samples to 10-shot per class but also the unlabeled target samples. And since MADA is the most SOTA and comparable method, we realize it with our fair 10-shot setting under both ViT-S and ResNet50 backbone on DomainNet and Office-Home datasets while maintaining other configurations of its report.
\newpage
\subsection{Data efficient test details in \ref{section:data_efficient} .}
\begin{table*}[htbp]
\centering
\caption{\textbf{Data Efficient Test Results on DomainNet.} The capital letters represent the target domains.}
\resizebox{\textwidth}{!}{
\begin{tabular}{lc c ccccccc c ccccccc}
\toprule
% \multirow{2}{*}{\makecell{\textbf{Method} \\ Second line}} &\multirow{2}{*}{\makecell{\textbf{Method} \\ Second line}} 
\multirow{2.5}{*}{\textbf{Method}} & \multirow{2.5}{*}{\textbf{Backbone}} && \multicolumn{7}{c}{\textbf{Data Counts (×10$^6$)}} && \multicolumn{7}{c}{\textbf{Training Consumed Time (×10$^4$ Second)}}\\ 
% \cline{3-9} \cline{10-14}
\cmidrule(lr){4-10} \cmidrule(lr){12-18}
 &&&C & I & P & Q & R & S & Avg&&C & I & P & Q & R & S & Avg\\ 
\midrule
% \textit{Supervised-10-shots:} &&&&&&&&&&&&&\\
% \multicolumn{14}{l}{\textbf{\textrm{Supervised-10-shots}}} \\
%8695400.0, 13391442.0, 8308180.0, 4024847.9999999995, 19759513.000000004, 21315960.0, 12582557.166666666//64040.0, 125240.00000000001, 62860.00000000001, 32220.000000000004, 131924.0, 181764.0, 99674.66666666667
MADA\cite{zhang2024revisiting_MADA}& ViT-S && 8.70 & 13.39 & 8.31 & 4.02 & 19.76 & 21.32 & 12.58 && 6.40 & 12.52 & 6.29 & 3.22 & 13.19 & 18.18 & 9.97 \\
% 7825860.0, 8207658.0, 4569498.999999999, 3689444.0, 4353791.0, 8777160.0, 6237235.333333333\\[75636.0, 95760.0, 45573.0, 40535.0, 42068.0, 95843.99999999999, 65902.66666666667]
MADA\cite{zhang2024revisiting_MADA}& ResNet50 && 7.83 & 8.21 & 4.57 & 3.69 & 4.35 & 8.78 & 6.24 && 7.56 & 9.58 & 4.56 & 4.05 & 4.21 & 9.58 & 6.59\\
% 2172100.0, 1420800.0, 1420700.0, 4195600.0, 1419400.0, 1897400.0, 2087666.6666666667]
%0.5595, 0.3661, 0.3665, 1.0903, 0.3591, 0.4894
\allsource{}& ViT-S && 2.17   & 1.42 & 1.42 & 4.20 & 1.42 & 1.90 & 2.09 && 0.56 & 0.37 & 0.39 & 1.09 & 0.36 & 0.49 &0.54\\
%1176000.0, 1149000.0, 890700.0, 966400.0, 1194300.0, 1158700.0, 1089183.3333333333 //1.1760, 1.1490, 0.8907, 0.9664, 1.1943, 1.1587//0.3541, 0.3288, 0.2784, 0.3363, 0.4708, 0.3580
\ourmethod{} (Ours)& ViT-S && 1.18 & 1.15 & 0.89 & 0.97  & 1.19 & 1.16 & 1.09 && 0.35 & 0.33 & 0.28 & 0.34  & 0.47 & 0.36 & 0.35 \\
\bottomrule
\end{tabular}
}
\label{tab:major2}
\end{table*}

\begin{table*}[!htbp]
\centering
\caption{\textbf{Data Efficient Log Scaled Test Results on DomainNet.} The capital letters represent the target domains.}
\resizebox{\textwidth}{!}{
\begin{tabular}{lc c ccccccc c ccccccc}
\toprule
% \multirow{2}{*}{\makecell{\textbf{Method} \\ Second line}} &\multirow{2}{*}{\makecell{\textbf{Method} \\ Second line}} 
\multirow{2.5}{*}{\textbf{Method}} & \multirow{2.5}{*}{\textbf{Backbone}} && \multicolumn{7}{c}{\textbf{Data Counts (log scale)}} && \multicolumn{7}{c}{\textbf{Training Consumed Time (log scale)}}\\ 
% \cline{3-9} \cline{10-14}
\cmidrule(lr){4-10} \cmidrule(lr){12-18}
 &&&C & I & P & Q & R & S & Avg&&C & I & P & Q & R & S & Avg\\ 
\midrule
% \textit{Supervised-10-shots:} &&&&&&&&&&&&&\\
% \multicolumn{14}{l}{\textbf{\textrm{Supervised-10-shots}}} \\
%15.97830471, 16.4101264 , 15.93275113, 15.2079977 , 16.7991456 ,16.87496665, 16.34782206 \\11.06726317, 11.73798718, 11.04866531, 10.38034266, 11.78998128,12.11046442, 11.50966683
MADA\cite{zhang2024revisiting_MADA}& ViT-S && 15.98 & 16.41 & 15.93 & 15.21 & 16.80 & 16.87 & 16.35 && 11.07 & 11.74 & 11.05 & 10.38 & 11.79 & 12.11 & 11.51 \\
% 15.87294419, 15.92057818, 15.33491413, 15.12098633, 15.28655752, 15.98766345, 15.64604759\\11.23368764, 11.46960034, 10.72707071, 10.60992108, 10.64704264, 11.47047715, 11.09593419
MADA\cite{zhang2024revisiting_MADA}& ResNet50 && 15.87 & 15.92 & 15.33 & 15.12 & 15.29 & 15.99 & 15.65 && 11.23 & 11.47 & 10.73 & 10.61 & 10.65 & 11.47 & 11.10\\
% 14.591205  , 14.16673065, 14.16666027, 15.24954692, 14.16574481, 14.45599509, 14.55155757\\8.62962862, 8.20549161, 8.20658361, 9.29679326, 8.18618599, 8.49576524, 8.5913416
\allsource{}& ViT-S && 14.59   & 14.17 & 14.18 & 15.25 & 14.17 & 14.46 & 14.55 && 8.63 & 8.21 & 8.21 & 9.30 & 8.19 & 8.50 & 8.59\\
%13.97762941, 13.95440256, 13.69976295, 13.78133311, 13.9930708 , 13.96280924, 13.90093874\\8.17216445, 8.09803476, 7.93164402, 8.12058871, 8.45701847, 8.18311808, 8.17301131
\ourmethod{} (Ours)& ViT-S && 13.98 & 13.95 & 13.70 & 13.78  & 13.99 & 13.96 & 13.90 && 8.17 & 8.10 & 7.93 & 8.12  & 8.46 & 8.18 & 8.17 \\
\bottomrule
\end{tabular}
}
\label{tab:major3}
\end{table*}

\subsection{Domain choosing analysis details.}
To compute the heatmap matrix visualizing domain preference in Figure \ref{fig:domain_choosing}(b), for each source domain, we count the samples selected from it until the target model converges under the 10-shot condition. Since the quantities of available samples varies significantly across domains, we normalize the counts by these quantities. The final domain preference is then determined by computing the importance of these normalized values.

Similarly, to compute the heatmap matrix for domain selection in Figure \ref{fig:domain_choosing}(a), we calculate the importance of the counts of selected samples from different source domains throughout the training epochs.

\newpage

\section{Method to get $s^*$ and $\underline{\alpha}^*$ which minimize \eqref{eq:multisourcetarget} in Theorem \ref{thm:multi_source}}\label{appendix:sa}

% The minimization problem of \eqref{eq:multisourcetarget} can be decomposed into two subproblems. The first subproblem is to minimize t, which is a $K\times K$ quadratic programming problem with respect to $\alpha$, and by resolving it, we can compute optimal $t^*$ and $\alpha^*$. The second subproblem is to find the optimal $s^*$ that minimizes the proposed measure, corresponding to our previous analysis of Theorem~\ref{thm:one_source} and Proposition~\ref{Proposition:one_source} when $s$ corresponds to $n_1$. After we get $s^*$ and $\alpha^*$, we can eventually get the optimal quantity $n_k^*$ of each source.

The minimization problem of proposed measure \eqref{eq:multisourcetarget} is
\begin{align}
(s^*, \underline{\alpha}^*) \gets \argmin\limits_{(s, \underline{\alpha})} \frac{d}{2}\left(\frac{1}{N_0+s}+\frac{s^2}{(N_0+s)^2}\frac{\underline{\alpha}^T\Theta^T{J}({{\underline{\theta}}_0})\Theta\underline{\alpha}}{d}\right).
\end{align}

We decompose this problem and explicitly formulate the constraints as follows.
\begin{align}
(s^*, \underline{\alpha}^*) \gets \argmin\limits_{s\in[0,\sum\limits_{i=1}^{K}N_i]} \frac{d}{2}\left(\frac{1}{N_0+s}+\frac{s^2}{(N_0+s)^{2}d}\argmin\limits_{\underline{\alpha}\in\cA(s)}\underline{\alpha}^T\Theta^T{J}({{\underline{\theta}}_0})\Theta\underline{\alpha}\right),
\end{align}
where 
\begin{align}
\cA(s)=\left\{\underline{\alpha}\Bigg|\sum_{i=1}^{K}\alpha_i=1, s*\alpha_i \le N_i, \alpha_i \ge0,i=1,\dots,K\right\}.
\end{align}
Due to the complex constraints between \( s \) and \( \alpha \), obtaining an analytical solution to this problem is challenging. Therefore, we propose a numerical approach to get the optimal solution.

This problem requires optimizing the objective function over two variables: a scalar variable $s$ representing the total transfer quantity, and a vector variable $\underline{\alpha}$ representing the proportion of samples drawn from each source domain. For $s$ which is restricted to integer values, we perform a exhaustive search over its feasible domain $[0, s_{\max}]$, where $s_{\max}= \sum\limits_{i=1}^{K} N_i$. For each candidate $s'$ in the search, we compute the optimal $\underline{\alpha}'$ under the constraint $\mathcal{A}(s')$, which is a $K\times K$ quadratic programming problem with respect to $\underline{\alpha}$

$$
\underline{\alpha}'=\arg\min\limits_{\underline{\alpha}\in\mathcal{A}(s')}\underline{\alpha}^T\Theta^T{J}({{\underline{\theta}}_0})\Theta\underline{\alpha}. 
$$

The quadratic coefficient matrix in this optimization problem is given by $\Theta^\top J(\underline{\theta}_0) \Theta$. Since the Fisher information matrix $J(\underline{\theta}_0)$ is positive semi-definite, the quadratic coefficient matrix is also positive semi-definite. This guarantees the existence of a global optimal solution. After getting $s'$ and $\underline{\alpha}'$, the function is then evaluated at $(s', \underline{\alpha}')$. 

In brief, for each $s'$, we solve for the corresponding optimal $\underline{\alpha}'$, yielding a finite collection of candidate solutions $(s', \underline{\alpha}')$ and their associated objective function values.
After completing the search, the optimal solution $(s^*, \underline{\alpha}^*)$ is chosen as the pair that achieves the lowest objective function values among all candidate $(s', \underline{\alpha}')$. Since the feasible set of $s$ is finite and enumerable, and for each fixed $s'$ the optimization over $\underline{\alpha}'$ has a solution, the overall optimization problem is guaranteed to have at least one global solution. Hence, the optimal pair $(s^*, \underline{\alpha}^*)$ exists.

It is worth noting that, for the sake of computational efficiency, we do not exhaustively enumerate all possible values of $s$ in our experiments. Instead, we perform a grid search using 1000 uniformly spaced steps over the feasible range of $s$, where the number of steps is denoted as $stepnumber = 1000$. Experimental results in Section \ref{sec:experi} demonstrate that this strategy does not compromise the effectiveness or stability of the method.
% Specifically, we perform a uniform-step exhaustive search over the feasible domain of \( s \). For each candidate \( s' \), we compute the optimal \( \alpha' \) under the constraint \( A(s') \), which is a $K\times K$ quadratic programming problem with respect to $\alpha$. The proposed measure \eqref{eq:multisourcetarget} value is then evaluated using \( (s', \alpha') \). After completing the search over \( s \), the optimal values \( s^* \) and \( \alpha^* \) correspond to the pair \( (s', \alpha') \) that minimizes the proposed measure.

\newpage
\section{Analysis on Domain-specific Transfer Quantity }
\label{Analysis_Quantity}
% Preference for Parameter-Similar Models
% For understanding the domain preferacnce of \ourmethod{}, we compute the importance of source domains for each target domain. For each source domain we first count all the samples selected from this source domain until target model converges under 10-shot condition, and given` that the quantities of available for selection samples from different domains differ significantly, we divide the counts by these quantities, and we get the final importance about domain preference after we compute the importance of these devidance results.
% For understanding the domain preferacnce of \ourmethod{}, 
% We compute the importance of source domains to represent the domain preference of \ourmethod{} for each target domain by the number of selected samples.

% We visualize the source domain preference of \ourmethod{} for each target domain based on the proportion of selected samples. As shown in Figure \ref{fig:domain_choosing}(b), when the target domain is Clipart, \ourmethod{} mainly leverages samples from Real, Painting and Sketch, and Quickdraw contributes minimally to any target domain, which aligns with the findings in \cite{peng2019moment_M3SDA}.
To understand the domain preference of \ourmethod{}, we visualize the proportion of each source in the selected samples for each target domain. As shown in Figure~\ref{fig:domain_choosing}(b), when the target domain is Clipart, \ourmethod{} primarily leverages samples from Real, Painting, and Sketch. In addition, Quickdraw contributes minimally to any target domain. These observations align with the findings in \cite{peng2019moment_M3SDA}. To further clarify the selection process of \ourmethod{} during training, we visualize it in Figure \ref{fig:domain_choosing}(a). 
We observe that in the \texttt{Office-Home} dataset, Clipart initially selects all source domains but later focuses on Art and Real World. In the \texttt{DomainNet} dataset, Sketch predominantly selects Clipart, Painting, and Real throughout the training process.

\begin{figure}[htbp]
    \centering
    % 左侧：两张图片垂直排列
    \begin{minipage}{0.45\linewidth}
        \centering
        % 第一张图片
        \includegraphics[width=\linewidth]{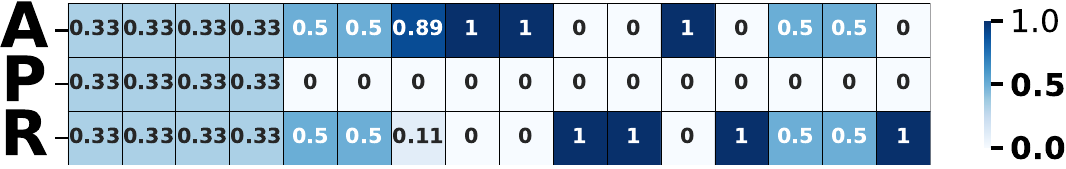}
        % \vspace{0.5cm} % 调整图片与下方的间距
        % 第二张图片
        \includegraphics[width=\linewidth]{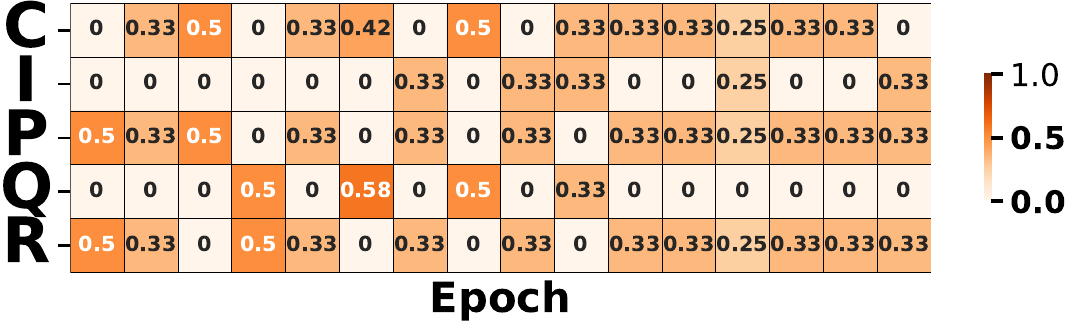}
        % \caption{Domain Selection}
        % \vspace{-0.7cm}
        \textbf{(a)} Domain Selection
        % \label{fig:second-image}
    \end{minipage}
    \hfill
    % 右侧：单独一张图片
    \begin{minipage}{0.45\linewidth}
        \centering
        \includegraphics[width=\linewidth]{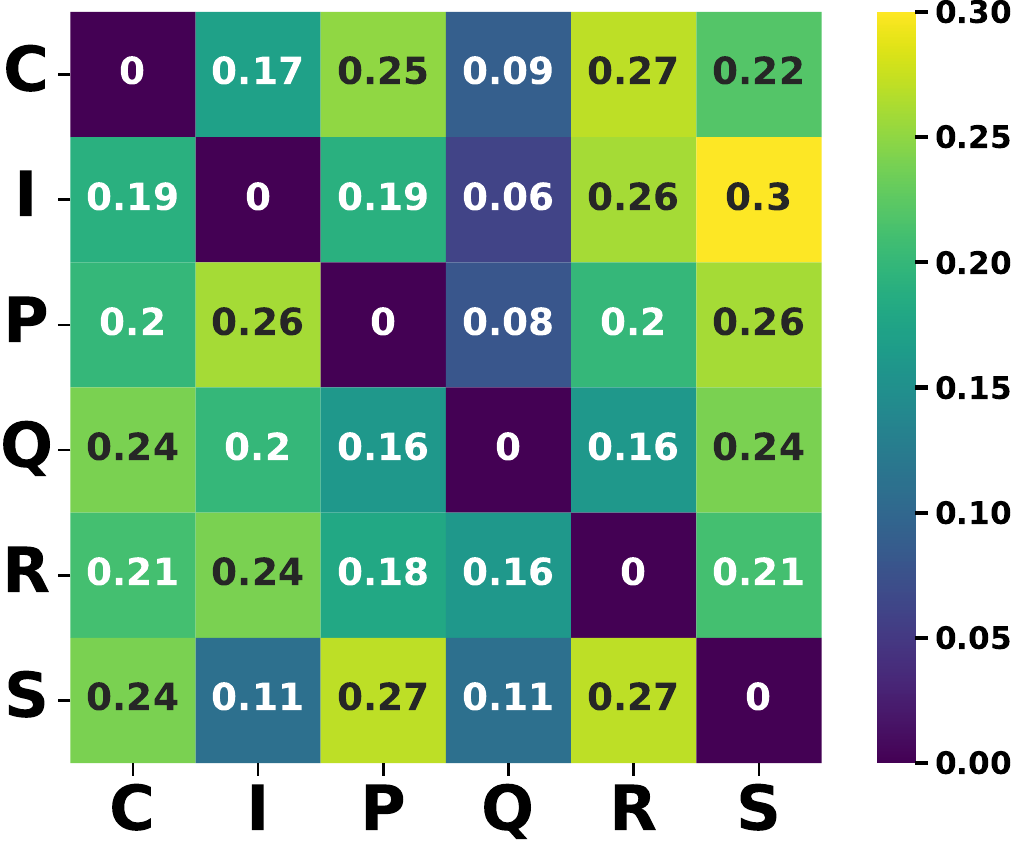}
        % \caption{Caption for third image}
        \textbf{(b)} Domain preference
        % \label{fig:third-image}
    \end{minipage}
    
    \caption{Visualization of domain-specific transfer quantity under 10-shot setting. (a) Domain selection during training epochs (from left to right), where the \textcolor{blue}{blue} upper part represents
    the selection of target domain Clipart on \texttt{Office-Home}, and the \textcolor{orange}{orange} lower part represents
    the selection of target domain Sketch on \texttt{DomainNet}. Darker colors indicate stronger tendencies throughout the training process. (b) Source domain preferences of different target domains on \texttt{DomainNet}. Each row corresponds to a target domain while each column represents a source domain.}
    \label{fig:domain_choosing}
\end{figure}

We additionally report the \textit{raw average transfer quantities} between domains on the Office-Home and DomainNet datasets. Each row denotes a target domain and each column denotes a source domain.

\begin{table}[htbp]
\centering
\caption{Average transfer quantities on the Office-Home dataset. Each row denotes a target domain and each column denotes a source domain.}
\begin{tabular}{c|cccc}
\toprule
Target $\backslash$ Source & A & C & P & R \\
\midrule
A & 0 & 2546 & 2956 & 3341 \\
C & 1325 & 0 & 753 & 2192 \\
P & 736 & 1320 & 0 & 2378 \\
R & 1879 & 3371 & 3300 & 0 \\
\bottomrule
\end{tabular}
\vspace{6pt}

\caption{Average transfer quantities on the DomainNet dataset. Each row denotes a target domain and each column denotes a source domain.}
\begin{tabular}{c|cccccc}
\toprule
Target $\backslash$ Source & C & I & P & Q & R & S \\
\midrule
C & 0 & 18391 & 57901 & 47641 & 138497 & 43100 \\
I & 28962 & 0 & 28979 & 17581 & 138497 & 41558 \\
P & 30907 & 41380 & 0 & 41945 & 83092 & 44356 \\
Q & 12380 & 11305 & 12883 & 0 & 30498 & 18050 \\
R & 24168 & 31066 & 36179 & 51749 & 0 & 27716 \\
S & 23591 & 11302 & 36897 & 50182 & 88133 & 0 \\
\bottomrule
\end{tabular}
\end{table}

\newpage
\section{A Detailed Discussion on the Limitations}
\label{limitations}
\begin{itemize}
    \item \textbf{Sampling Method.}  Since our theoretical analysis focuses on the transfer quantity from each source task, we adopt a straightforward random sampling strategy in the algorithm implementation. Given that our theoretical results are derived under average-case assumptions, random sampling is sufficient to demonstrate the robustness of both our theoretical analysis and the proposed algorithm. Nevertheless, we anticipate that more sophisticated sampling strategies, such as active sampling, may further improve the algorithm’s performance.

    \item \textbf{Weight or Quantity.} Many existing work of multi-source transfer learning assign weights to source tasks and utilize all available samples. In contrast, this work focuses on optimizing the transfer quantity from each source task. We anticipate that future work could further improve algorithmic performance by jointly optimizing both the sample weights and the transfer quantity from each source task.

     \item \textbf{Possible Extensions to Other Loss Functions.} Our theoretical analysis is developed under the assumption of a negative log-likelihood objective, which aligns with the cross-entropy loss commonly used in classification tasks with softmax outputs and one-hot labels. However, for other learning objectives such as the mean squared error in regression problems, our framework does not yet provide a direct theoretical guarantee. We believe, nevertheless, that the core idea can be generalized to broader loss functions under suitable regularity conditions. We leave a more rigorous theoretical extension and empirical validation with alternative loss functions as future work to demonstrate the robustness and wider applicability of our method.

\end{itemize}

\section{Broader Impacts}
\label{impacts}

We first develop a theoretical framework for optimizing transfer quantities in transfer learning, and subsequently propose an architecture-agnostic and data-efficient algorithm based on this theoretical framework. The proposed  theoretical framework and algorithm have broad applicability in various transfer learning scenarios, including domains such as medical image analysis, recommendation systems, and anomaly detection. There are no negative social impacts foreseen.

%%%%%%%%%%%%%%%%%%%%%%%%%%%%%%%%%%%%%%%%%%%%%%%%%%%%%%%%%%%%

\end{document}